\documentclass[11pt]{article}

\usepackage{amsmath,amsthm,amssymb,amsfonts,bbm,bm}
\usepackage{cite,epsfig,epsf,url,enumitem}
\usepackage{etoolbox}
\usepackage{graphicx}
\usepackage{hyperref}
\usepackage{fullpage}

\usepackage[small,bf]{caption}
\newcommand\EnumPrefix{}

\newlist{senenum}{enumerate}{10}
\setlist[senenum]{label=\arabic*.,ref=\EnumPrefix,leftmargin=*}

\AtBeginEnvironment{lemmalist}{\renewcommand\EnumPrefix{\thelemmalist.\arabic*}}

\newtheorem{theorem}{Theorem}[section]
\newtheorem{lemma}[theorem]{Lemma}
\newtheorem{lemmalist}[theorem]{Lemma}

\newtheorem{corollary}[theorem]{Corollary}
\newtheorem{proposition}[theorem]{Proposition}
\newtheorem{definition}[theorem]{Definition}

\newcommand{\R}{\mathbb{R}}

\renewcommand{\P}{\operatorname{\mathbb{P}}}
\newcommand{\E}{\operatorname{\mathbb{E}}}

\newcommand{\1}[1]{\mathbf{1}\!\left[#1\right]}

\newcommand{\eqdef}{\stackrel{\text{def}}{=}}

\numberwithin{equation}{section}

\newcommand{\remove}[1]{}
\newcommand{\citeyr}[1]{\cite{#1}}

\newcommand{\submit}[2]{\ifdefined\issubmit{#1}\else{#2} \fi }

\title{Train faster, generalize better:\\
  Stability of stochastic gradient descent}

\author{
  Moritz Hardt\thanks{Email: mrtz@google.com} \and
  Benjamin Recht\thanks{Email: brecht@berkeley.edu,
    work performed at Google.} \and
  Yoram Singer\thanks{Email: singer@google.com}}

\begin{document}

\maketitle

\begin{abstract}
We show that parametric models trained by a stochastic gradient method (SGM)
with few iterations have vanishing generalization error. We prove our results
by arguing that SGM is algorithmically stable in the sense of Bousquet and
Elisseeff. Our analysis only employs elementary tools from convex and
continuous optimization. We derive stability bounds for both convex and
non-convex optimization under standard Lipschitz and smoothness assumptions.

Applying our results to the convex case, we provide new insights for why
multiple epochs of stochastic gradient methods generalize well in practice.
In the non-convex case, we give a new interpretation of common practices in
neural networks, and formally show that popular techniques for training large
deep models are indeed stability-promoting. Our findings conceptually
underscore the importance of reducing training time beyond its obvious
benefit.

\end{abstract}

\section{Introduction}

The most widely used optimization method in machine learning practice is
\emph{stochastic gradient method} (SGM). Stochastic gradient methods aim to
minimize the empirical risk of a model by repeatedly computing the gradient of
a loss function on a single training example, or a batch of few examples, and
updating the model parameters accordingly. SGM is scalable, robust, and
performs well across many different domains ranging from smooth and strongly
convex problems to complex non-convex objectives.

In a nutshell, our results establish that:
\submit{
\emph{Any model trained with stochastic gradient method
in a reasonable amount of time attains small generalization error.}
}{
\begin{center}
\emph{Any model trained with stochastic gradient method
in a reasonable\\ amount of time attains small generalization error.}
\end{center}
} 

As training time is inevitably limited in practice, our results help to
explain the strong generalization performance of stochastic gradient methods
observed in practice. More concretely, we bound the generalization error of
a model in terms of the number of iterations that stochastic gradient method
took in order to train the model. Our main analysis tool is to employ the
notion of algorithmic stability due to Bousquet and
Elisseeff~\citeyr{BousquetE02}. We demonstrate that the stochastic
gradient method is stable provided that the objective is relatively smooth
and the number of steps taken is sufficiently small.

It is common in practice to perform a linear number of steps in the size of
the sample and to access each data point multiple times. Our results show in
a broad range of settings that, provided the number of iterations is linear
in the number of data points, the generalization error is bounded by a
vanishing function of the sample size. The results hold true even for
complex models with large number of parameters and no explicit
regularization term in the objective. Namely, fast training time by itself
is sufficient to prevent overfitting.

Our bounds are algorithm specific: Since the number of iterations we allow
can be larger than the sample size, an arbitrary algorithm could easily
achieve small training error by memorizing all training data with no
generalization ability whatsoever. In contrast, if the stochastic gradient
method manages to fit the training data in a reasonable number of
iterations, it is guaranteed to generalize.

Conceptually, we show that minimizing training time is not only beneficial
for obvious computational advantages, but also has the important
byproduct of decreasing generalization error. Consequently, it may make
sense for practitioners to focus on minimizing training time, for instance,
by designing model architectures for which stochastic gradient method
converges fastest to a desired error level.

\subsection{Our contributions}

Our focus is on generating \emph{generalization bounds} for models learned with stochastic gradient descent.  Recall that the generalization bound is the expected difference between the error a model incurs on a training set versus the error incurred on a new data point, sampled from the same distribution that generated the training data.  Throughout, we assume we are training models using $n$ sampled data points.

Our results build on a fundamental connection between the generalization error
of an algorithm and its \emph{stability} properties.  Roughly speaking, an
algorithm is \emph{stable} if the training error it achieves varies only
slightly if we change any single training data point. The precise notion of
stability we use is known as \emph{uniform stability} due
to~\cite{BousquetE02}. It states that a randomized algorithm~$A$ is uniformly
stable if for all data sets differing in only one element, the learned models produce nearly the same predictions.   We review this method in Section~\ref{updates:sec}, and provide a new adaptation of this theory to \emph{iterative} algorithms.

In Section~\ref{sec:sgd}, we show that stochastic gradient is uniformly stable, and our techniques mimic its convergence proofs.  For convex loss functions, we prove that the stability measure decreases as a function of the sum of the step sizes.  For strongly convex loss functions, we show that stochastic gradient is stable, even if we train for an arbitrarily long time.   We can combine our bounds on the generalization error of stochastic gradient
method with optimization bounds quantifying the convergence of the empirical
loss achieved by SGM. In Section~\ref{sec:case-studies}, we show that models trained for multiple epochs match classic bounds for stochastic gradient~\cite{NY78,NY83}.  
%
%

More surprisingly, our results carry over to the case where the loss-function is non-convex.  In this case we show that the method generalizes provided the steps are sufficiently small and the number of iterations is not too large.  More specifically, we show the number of steps of stochastic gradient can grow as $n^c$ for a small $c>1$.  This provides some explanation as to why neural networks can be trained for  multiple epochs of stochastic gradient and still exhibit excellent generalization.
In Section~\ref{sec:operations}, we furthermore show that various heuristics used in practice, especially in the deep learning community, help to increase the stability of stochastic gradient method. 
For example, the popular dropout scheme~\cite{krizhevsky2012imagenet,srivastava2014dropout} 
improves all of our bounds. Similarly, $\ell_2$-regularization improves the exponent of $n$ in our non-convex result. In
fact, we can drive the exponent arbitrarily close to $1/2$ while preserving the
non-convexity of the problem.

\subsection{Related work}

There is a venerable line of work on stability and generalization dating back
more than thirty
years~\cite{DevroyeW79,KearnsR99,BousquetE02,MukherjeeNPR06,SSSS10}.  The
landmark work by Bousquet and Elisseeff~\cite{BousquetE02} introduced the
notion of \emph{uniform stability} that we rely on. They showed that several
important classification techniques are uniformly stable. In particular, under certain regularity assumptions, it
was shown that the \emph{optimizer} of a regularized empirical loss
minimization problem is uniformly stable.  Previous work generally applies only to
the exact minimizer of specific optimization problems.  It is not immediately evident on how to compute a generalization bound for an approximate minimizer such as one found by using stochastic gradient.  Subsequent work studied stability bounds for randomized algorithms but focused on random perturbations of the cost function, such as those induced by bootstrapping or bagging~\cite{Elisseeff05}.  This manuscript differs from this foundational work in that it derives stability bounds about the learning \emph{procedure}, analyzing algorithmic properties that induce stability.

Stochastic gradient descent, of course, is closely related to our inquiry.  
Classic results by Nemirovski and Yudin show that the stochastic gradient 
method produces is nearly optimal for empirical risk minimization of convex loss functions~\cite{NY78,  NY83, Nemirovski09,frostig15competing}.  These results have been extended by many machine learning researchers, yielding tighter bounds and probabilistic guarantees ~\cite{Hazan06,HazanKale14, Rakhlin11}.  However, there is an important limitation of all of this prior art.  The derived generalization bounds only hold for single passes over the data.  That is, in order for the bounds to be valid, each training example must be used no more than once in a stochastic gradient update.  In practice, of course, one tends to run multiple \emph{epochs} of the stochastic gradient method.	Our results resolve this issue by combining stability with optimization error.  We use the foundational results to estimate the error on the \emph{empirical risk} and then use stability to derive a deviation from the true risk.  This enables us to study the risk incurred by multiple epochs and provide simple analyses of regularization methods for convex stochastic gradient.  We compare our results to this related work in Section~\ref{sec:case-studies}.  We note that Rosasco and Villa obtain risk bounds for least squares minimization with an incremental gradient method in terms of the number of epochs~\cite{Rosasco14}. These bounds are akin to 
our study in Section~\ref{sec:case-studies}, 
although our results are incomparable due to various different assumptions.

Finally, we note that in the non-convex case, the stochastic gradient method is remarkably successful for training large neural networks~\cite{Bottou98, krizhevsky2012imagenet}.  However, our theoretical understanding of this method is limited.  Several authors have shown that the stochastic gradient method finds a stationary point of nonconvex cost functions~\cite{KushnerBook,ghadimi2013stochastic}.  Beyond asymptotic convergence to stationary points, little is known about finding models with low training or generalization error in the nonconvex case.   There have recently been several important studies investigating optimal training of neural nets.  For example Livni \emph{et al.} show that networks with polynomial activations can be learned in a greedy fashion~\cite{livni2014computational}.  Janzamin \emph{et al.}~\cite{Janzamin15} show that two layer neural networks can be learned using tensor methods. Arora~\emph{et al.}~\cite{Arora15} show that two-layer sparse coding dictionaries can be learned via stochastic gradient. Our work complements these developments: rather than providing new insights into mechanisms that yield low training error, we provide insights into mechanisms that yield low generalization error. If one can achieve low training error quickly on a nonconvex problem with stochastic gradient, our results guarantee that the resulting model generalizes well.

\section{Stability of randomized iterative algorithms} \label{updates:sec}

Consider the following general setting of supervised learning. There is an
unknown distribution~${\cal D}$ over examples from some
space $Z.$ We receive a sample $S=(z_1,\dots,z_n)$ of $n$ examples drawn
i.i.d.~from ${\cal D}.$
Our goal is to find a model~$w$ with small \emph{population risk}, defined as:
\submit{$R[w]\eqdef \E_{z\sim{\cal D}} f(w;z)\,.$}
{\[
R[w]\eqdef \E_{z\sim{\cal D}} f(w;z)\,.
\]
} 
Here, where $f$ is a \emph{loss function} and $f(w;z)$ designates the
\emph{loss} of the model described by~$w$ encountered on example~$z$.

Since we cannot measure the objective $R[w]$ directly, we instead use a sample-averaged proxy, the \emph{empirical risk}, defined as
\submit{
$R_S[w]\eqdef\frac1n\sum_{i=1}^n f(w;z_i)\,,$
}{
\[
R_S[w]\eqdef\frac1n\sum_{i=1}^n f(w;z_i)\,,
\]
} 

The \emph{generalization error} of a model~$w$ is the difference
\begin{equation}\label{eq:gen-error}
R_S[w]-R[w].
\end{equation}
When $w=A(S)$ is chosen as a function of the data by a potentially randomized
algorithm~$A$ it makes sense to consider the expected generalization error
\begin{equation}
\epsilon_{\mathrm{gen}}\eqdef \E_{S,A}[R_S[A(S)] - R[A(S)]]\,,
\end{equation}
where the expectation is over the randomness of $A$ and the sample~$S.$

In order to bound the generalization error of an algorithm,  we employ the
following notion of \emph{uniform stability} in which we allow randomized
algorithms as well.

\begin{definition}
A randomized algorithm $A$ is $\epsilon$-\emph{uniformly stable} if
for all data sets $S,S'\in Z^n$ such that $S$ and $S'$ differ in at most one
example, we have
\begin{equation}\label{eq:stab}
\sup_{z} \E_{A} \left[ f(A(S); z) - f(A(S'); z) \right] \le \epsilon\,.
\end{equation}
Here, the expectation is taken only over the internal randomness of $A.$
We will denote by $\epsilon_{\mathrm{stab}}(A,n)$ the infimum over all
$\epsilon$ for which \eqref{eq:stab} holds. We will omit the tuple $(A,n)$ when
it is clear from the context.
\end{definition}

We recall the important theorem that uniform stability implies
\emph{generalization in expectation}.
\submit{}{Since our notion of stability differs slightly from existing ones
with respect to the randomness of the algorithm, we include a proof for the
sake of completeness.} The proof is based on an argument
in Lemma~7 of~\cite{BousquetE02} and very similar to Lemma~11 in~\cite{SSSS10}.

\begin{theorem}\submit{}{[Generalization in expectation]}
\label{thm:stab2gen}
Let $A$ be $\epsilon$-uniformly stable. Then,
\submit{
$\left|\E_{S,A}\left[ R_S[A(S)]-R[A(S)]\right]\right| \le \epsilon\,.$
}{
\[
\left|\E_{S,A}\left[ R_S[A(S)]-R[A(S)]\right]\right| \le \epsilon\,.
\]
} 
\end{theorem}
\submit{}{
\begin{proof}
Denote by $S=(z_1,\dots,z_n)$ and $S'=(z_1',\dots,z_n')$ two independent
random samples and let $S^{(i)}=(z_1,\dots,z_{i-1},z_i',z_{i+1},\dots,z_n)$ be
the sample that is identical to $S$ except in the $i$'th example where we
replace $z_i$ with $z_i'$. With this notation, we get that
\begin{align*}
\E_S\E_A \left[ R_S[A(S)]\right]
& = \E_S\E_A \left[ \frac 1n\sum_{i=1}^n f(A(S);z_i) \right]\\
& = \E_S\E_{S'}\E_A\left[\frac 1n\sum_{i=1}^n f(A(S^{(i)});z_i')\right] \\
& = \E_S\E_{S'}\E_A\left[\frac 1n\sum_{i=1}^n f(A(S);z_i')\right] + \delta \\
& = \E_S\E_A\left[ R[A(S)]\right] + \delta,
\end{align*}
where we can express $\delta$ as
\[
\delta
=
\E_S\E_{S'}\E_A\left[\frac 1n\sum_{i=1}^n f(A(S^{(i)});z_i')
- \frac 1n\sum_{i=1}^n f(A(S);z_i')\right]\,.
\]
Furthermore, taking the supremum over any two data sets $S,S'$ differing in only
one sample, we can bound the difference as
\[
\left|\delta\right| \le \sup_{S,S',z}\E_A\left[ f(A(S);z) - f(A(S');z)\right] \le \epsilon,
\]
by our assumption on the uniform stability of $A.$ The claim follows.
\end{proof}
} 

Theorem~\ref{thm:stab2gen} proves that if an algorithm is uniformly stable, then
its generalization error is small. We now turn to some properties of iterative
algorithms that control their uniform stability.

\if{0}
\begin{lemma}
Let $A$ be $\epsilon$-uniformly stable. Then, for every $t>0,$
\[
\P_S\left\{
\E_{A}F_S(A(S))\ge \E_A F(A(S)) + t\epsilon \right\}
\le \frac1t\,.
\]
\end{lemma}
\begin{proof}
The proof directly follows from Markov's inequality and
Theorem~\ref{thm:stab2gen}.
\end{proof}
\fi

\subsection{Properties of update rules}

We consider general update rules of the form $G\colon \Omega\to \Omega$
which map a point $w\in \Omega$ in the parameter space to another point
$G(w).$ The most common update is the gradient update rule
\submit{
$G(w) = w - \alpha \nabla f(w) ~ ,$
}{
\[
G(w) = w - \alpha \nabla f(w) ~ ,
\]
} 
where $\alpha \ge 0$ is a step size and $f\colon \Omega\to\mathbb{R}$
is a function that we want to optimize.

The canonical update rule we will consider in this manuscript is an
incremental gradient update, where $G(w) = w-\alpha \nabla f(w)$ for some
convex function $f$. We will return to a detailed discussion of this specific
update in the sequel, but the reader should keep this particular example in
mind throughout the remainder of this section.

The following two definitions provide the foundation of our analysis of how
two different sequences of update rules diverge when iterated from the same
starting point. These definitions will ultimately be useful when analyzing
the stability of stochastic gradient descent.

\submit{
\begin{definition} \label{expansive:def} \label{boundedness:def}
An update rule is {$\eta$-expansive} if for all $v,w\in\Omega,$
$\left\| G(v) - G(w)\right\| 
\le \eta\|v-w\|.$ It is \emph{$\sigma$-bounded} if
$\left\| w - G(w)\right\| \le \sigma\,.$
\end{definition}
}{
\begin{definition} \label{expansive:def}
An update rule is {$\eta$-expansive} if
\begin{equation}
\sup_{v,w\in \Omega}\frac{\left\| G(v) - G(w)\right\|}{\|v-w\|} \le \eta\,.
\end{equation}
\end{definition}
\begin{definition} \label{boundedness:def}
An update rule is \emph{$\sigma$-bounded} if
\begin{equation}
\sup_{w\in \Omega}\left\| w - G(w)\right\| \le \sigma\,.
\end{equation}
\end{definition}
}

With these two properties, we can establish the following lemma of how a sequence of updates 
to a model diverge when the training set is perturbed.  
\begin{lemma}[Growth recursion]
\label{lem:growth}
Fix an arbitrary sequence of updates $G_1,\dots,G_T$
and another sequence $G_1',\dots,G_T'.$
Let $w_0=w_0'$ be a starting point in $\Omega$ and define
$\delta_t = \|w_t' - w_t\|$ where $w_t,w'_t$ are defined recursively through
\begin{align*}
w_{t+1}  = G_t(w_t)\qquad w_{t+1}' &= G_t'(w_t')\,. \tag{$t>0$}
\end{align*}
Then, we have the recurrence relation 
\submit{$\delta_0=0$,
\begin{align*}
\delta_{t+1} & \le \begin{cases}
\eta\delta_{t} & \text{$G_t=G_t'$ is $\eta$-expansive}  \\
\min(\eta,1) \delta_{t}+2\sigma_t & \text{$G_t$ and $G_t'$ are $\sigma$-bounded,} \\
 \, & \text{$G_t$ is $\eta$ expansive} 
\end{cases}
\end{align*}
}{
\begin{align*}
\delta_0 & = 0 \\
\delta_{t+1} & \le \begin{cases}
\eta\delta_{t} & \text{$G_t=G_t'$ is $\eta$-expansive}  \\
\min(\eta,1) \delta_{t}+2\sigma_t & \text{$G_t$ and $G_t'$ are $\sigma$-bounded,} \\
 \, & \text{$G_t$ is $\eta$ expansive} \\
\end{cases}\tag{$t>0$}
\end{align*}
} 
\end{lemma}
\submit{}{
\begin{proof}
The first bound on $\delta_t$ follow directly from the assumption that
$G_t=G_t'$ and the definition of expansiveness. For the second bound, recall
from Definition~\ref{boundedness:def} that if $G_t$ and $G_t'$ are
$\sigma$-bounded, then by the triangle inequality,
\begin{align*}
\delta_{t+1}
	&= \|G(w_{t}) - G'(w_t')\| \\
	&\le \|G(w_{t})-w_t + w_t' - G'(w_t')\| + \|w_t - w_t'\| \\
	&\le \delta_t + \|G(w_{t})-w_{t}\| + \|G(w_t')-w_t'\| \\
	&\le \delta_t + 2\sigma\,,
\end{align*}
which gives half of the second bound.
We can alternatively bound $\delta_{t+1}$ as
\begin{align*}
\delta_{t+1} &= \|G_t(w_{t}) - G_t'(w_{t}')\|\\
& = \|G_t(w_{t}) - G_t(w_{t}') + G_t(w_{t}') - G_t'(w_{t}')\|\\
& \leq \|G_t(w_{t}) - G_t(w_{t}') \| + \|G_t(w_{t}') - G_t'(w_{t}')\|\\
& \leq \|G_t(w_{t}) - G_t(w_{t}') \| + \|w_{t}'-G_t(w_{t}')\|+\|w_{t}'- G_t'(w_{t}')\|\\
& \leq \eta \delta_t + 2 \sigma\,.
\end{align*}
\end{proof}
} 


\section{Stability of Stochastic Gradient Method}\label{sec:sgd}

Given $n$ labeled examples $S=(z_1,\dots,z_n)$ where $z_i\in Z,$ consider a
\emph{decomposable} objective function
\submit{$f(w) = \frac 1n\sum_{i=1}^n f(w;z_i),$}{
\[
f(w) = \frac 1n\sum_{i=1}^n f(w;z_i),
\]
} 
where $f(w;z_i)$ denotes the \emph{loss} of $w$ on the
example~$z_i.$ The stochastic gradient update for this problem with \emph{learning rate} $\alpha_t>0$ is given by
\submit{$w_{t+1} = w_t - \alpha_t \nabla_w f(w_t;z_{i_t}) \,.$}{
\[
w_{t+1} = w_t - \alpha_t \nabla_w f(w_t;z_{i_t}) \,.
\]
} 
Stochastic gradient method (SGM) is the algorithm resulting from performing
stochastic gradient updates $T$ times where the indices $i_t$ are randomly
chosen. There are two popular schemes for choosing the examples' indices. One
is to pick $i_t$ uniformly at random in $\{1,\dots,n\}$ at each step. The
other is to choose a random permutation over $\{1,\dots,n\}$ and cycle through
the examples repeatedly in the order determined by the permutation. Our
results hold for both variants.

In parallel with the previous section the stochastic gradient method is akin to applying the \emph{gradient update rule} defined as follows.

\begin{definition}
For a nonnegative step size $\alpha\ge 0$ and a function $f\colon \Omega\to\mathbb{R},$
we define the \emph{gradient update rule} $G_{f,\alpha}$ as
\submit{$G_{f,\alpha}(w) = w - \alpha \nabla f(w)\,.$}{
\[
G_{f,\alpha}(w) = w - \alpha \nabla f(w)\,.
\]
} 
\end{definition}

\subsection{Proof idea: Stability of stochastic gradient method}

In order to prove that the stochastic gradient method is stable, we will analyze the output of the algorithm on two data sets that differ in precisely one location.
Note that if the loss function is $L$-Lipschitz for every example~$z$,
we have $\E|f(w;z)-f(w';z)|\le L
\E\|w-w'\|$ for all $w$ and $w'$.
Hence, it suffices to analyze how $w_t$ and $w_t'$ diverge
in the domain as a function of time~$t.$ Recalling that $w_t$ is obtained from
$w_{t-1}$ via a gradient update, our goal is to bound $\delta_t=\|w_t-w_t'\|$
recursively and in expectation as a function of $\delta_{t-1}.$

There are two cases to consider. In the first case, SGM selects the index of an
example at step $t$ on which is identical in $S$ and $S'$.  Unfortunately, it
could still be the case that $\delta_t$ grows, since $w_t$ and $w_t'$ differ
and so the gradients at these two points may still differ.  Below, we will show
how to control $\delta_t$ in terms of the convexity and smoothness properties
of the stochastic gradients.

The second case to consider is when SGM selects the one example to update in
which $S$ and $S'$ differ. Note that this happens only with probability~$1/n$ if
examples are selected randomly. In this case, we simply bound the increase in
$\delta_t$ by the norm of the two gradient $\nabla f(w_{t-1};z)$ and $\nabla
f(w_{t-1}';z').$ The sum of the norms is bounded by $2\alpha_t L$ and we obtain
$\delta_t\le\delta_t + 2\alpha_tL.$
Combining the two cases, we can then solve a simple recurrence relation to
obtain a bound on $\delta_T.$

This simple approach suffices to obtain the desired result in the convex case,
but there are additional difficulties in the non-convex case. Here, we need to
use an intriguing stability property of stochastic gradient method.
Specifically, the first time step~$t_0$ at which SGM even encounters the
example in which $S$ and $S'$ differ is a random variable in $\{1,\dots,n\}$
which tends to be relatively large. Specifically, for any $m\in\{1,\dots,n\},$
the probability that $t_0 \le m$ is upper bounded by $m/n.$ This allows us to
argue that SGM has a long ``burn-in period'' where $\delta_t$ does not grow at
all. Once $\delta_t$ begins to grow, the step size has already decayed
allowing us to obtain a non-trivial bound.

We now turn to making this argument precise.

\subsection{Expansion properties of stochastic gradients}
Let us now  record some of the core properties of the stochastic gradient update. The gradient update rule is bounded provided that the function~$f$ satisfies the
following common Lipschitz condition.

\begin{definition}
We say that $f$ is \emph{$L$-Lipschitz} if for all points $u$ in the domain of
$f$ we have $\|\nabla f(x)\|\le L.$ This implies that
\submit{
$|f(u) - f(v)| \le L\|u-v\| \,.$
}{
\begin{equation}
|f(u) - f(v)| \le L\|u-v\| \, .
\end{equation}
} 
\end{definition}

\begin{lemma}
\label{lem:Lipschitz2bounded}
Assume that $f$ is $L$-Lipschitz. Then, the gradient update
$G_{f,\alpha}$ is  $(\alpha L)$-bounded.
\end{lemma}

\submit{}{
\begin{proof}
By our Lipschitz assumption,
$
\|w - G_{f,\alpha}(w)\|
 = \|\alpha \nabla f(w)\|
 \le \alpha L\,$ .
\end{proof}
} 

We now turn to expansiveness.  As we will see shortly, different expansion
properties are achieved for non-convex, convex, and strongly convex functions.
\submit{}{
\begin{definition}
A function $f\colon \Omega\to\mathbb{R}$ is \emph{convex} if  for all
$u,v\in \Omega$ we have
\[
f(u) \ge f(v) + \langle \nabla f(v),u-v\rangle\,.
\]
\end{definition}
} 
\begin{definition}
A function $f\colon \Omega\to\mathbb{R}$ is \emph{$\gamma$-strongly convex} if
for all $u,v\in \Omega$ we have
\submit{
$f(u) \ge f(v)+ \langle \nabla f(v),u-v\rangle + \frac\gamma2 \|u-v\|^2\,.$
}{
\[
f(u) \ge f(v)+ \langle \nabla f(v),u-v\rangle + \frac\gamma2 \|u-v\|^2\,.
\]
} 
\end{definition}
\submit{We say $f$ is \emph{convex} if it is $0$-strongly convex.}{}
The following standard notion of smoothness leads to a bound on how
expansive the gradient update is.
\begin{definition}
A function $f\colon \Omega\to\mathbb{R}$ is
\emph{$\beta$-smooth} if for all for all $u,v\in \Omega$ we have
\submit{
$\|\nabla f(u) - \nabla f(v)\| \le \beta \|u-v\| \,.$
}{
\begin{equation}
\|\nabla f(u) - \nabla f(v)\| \le \beta \|u-v\| \,.
\end{equation}
} 
\end{definition}

In general, smoothness will imply that the gradient updates cannot be overly
expansive. When the function is also convex and the step size is sufficiently
small the gradient update becomes non-expansive. When the function is
additionally strongly convex, the gradient update becomes \emph{contractive}
in the sense that $\eta$ will be less than one and $u$ and $v$ will actually
shrink closer to one another. The majority of the following results can be
found in several textbooks and monographs. Notable references are
Polyak~\cite{PolyakBook} and Nesterov~\cite{NesterovBook}. We include proofs
in the appendix for completeness.

\begin{lemmalist}
Assume that $f$ is $\beta$-smooth. 
\submit{ 
\label{lem:smooth2expansive}
\label{lem:convex2expansive}
\label{lem:sconvex2expansive}
Then,
$G_{f,\alpha}$ is  $(1+\alpha\beta)$-expansive.
If $f$ is in addition convex, then for any $\alpha \le 2/\beta,$ the update $G_{f,\alpha}$ is $1$-expansive. If $f$ is in addition $\gamma$-strongly convex, then for $\alpha\leq\frac{2}{\beta+\gamma}$,
$G_{f,\alpha}$ is  $\left(1-\frac{\alpha\beta\gamma}{\beta+\gamma}\right)$-expansive.
}{
Then, the following properties hold.
\begin{senenum}
\item \label{lem:smooth2expansive}
$G_{f,\alpha}$ is  $(1+\alpha\beta)$-expansive.

\item \label{lem:convex2expansive}
Assume in addition that $f$ is convex. Then, for any $\alpha \le 2/\beta,$ the gradient update $G_{f,\alpha}$ is $1$-expansive.

\item \label{lem:sconvex2expansive}
Assume in addition that $f$ is $\gamma$-strongly convex. Then, for $\alpha\leq\frac{2}{\beta+\gamma}$,
$G_{f,\alpha}$ is  $\left(1-\frac{\alpha\beta\gamma}{\beta+\gamma}\right)$-expansive.
\end{senenum}
}
\end{lemmalist}

Henceforth we will no longer mention which random selection rule we use
as the proofs are almost identical for both rules.

\subsection{Convex optimization}

We begin with a simple stability bound for convex loss minimization via
stochastic gradient method.

\begin{theorem}
\label{thm:convex}
Assume that the loss function $f(\cdot\,;z)$ is $\beta$-smooth, convex
and $L$-Lipschitz for every~$z.$
Suppose that we run SGM with step
sizes $\alpha_t\le 2/\beta$ for $T$ steps.
Then, SGM satisfies uniform stability with
\submit{
$\epsilon_{\mathrm{stab}} \le \frac{2L^2}{n}\sum_{t=1}^T\alpha_t\,.$
}{
\[
\epsilon_{\mathrm{stab}} \le \frac{2L^2}{n}\sum_{t=1}^T\alpha_t\,.
\]
} 
\end{theorem}

\begin{proof}
Let $S$ and $S'$ be two samples of size $n$ differing in only a single
example. Consider the gradient updates $G_1,\dots,G_T$ and $G_1',\dots,G_T'$
induced by running SGM on sample $S$ and $S',$ respectively. Let $w_T$ and
$w_T'$ denote the corresponding outputs of SGM.

We now fix an example $z\in Z$ and apply
the Lipschitz condition on $f(\cdot\,;z)$ to get
\begin{equation}\label{eq:convex-diff}
\E\left|f(w_T;z)-f(w_T';z)\right| \le
L\E\left[\delta_T\right]\,,
\end{equation}
where $\delta_T=\|w_T-w_T'\|.$
Observe that at step $t,$ with probability $1-1/n,$ the
example selected by SGM is the same in both $S$ and $S'.$ In this case we have
that $G_t=G_t'$ and we can use the $1$-expansivity of the update rule $G_t$
which follows from Lemma~\ref{lem:convex2expansive} using the fact that the
objective function is convex and that $\alpha_t\le 2/\beta$.
With probability $1/n$ the selected example is
different in which case we use that both $G_t$ and $G_t'$ are
$\alpha_tL$-bounded as a consequence of Lemma~\ref{lem:Lipschitz2bounded}.
Hence, we can apply Lemma~\ref{lem:growth} and linearity of expectation to
conclude that for every $t,$
\submit{
$\E\left[\delta_{t+1}\right]
 \le \left(1-\frac1n\right)\E\left[\delta_t\right] +
 \frac1n\E\left[\delta_t\right]  +
\frac{2\alpha_t L}n  = \E\left[\delta_t\right] + \frac{2L\alpha_t}{n}\,.$
}{
\begin{align}\label{eq:convex-recursion}
\E\left[\delta_{t+1}\right]
 \le \left(1-\frac1n\right)\E\left[\delta_t\right] +
 \frac1n\E\left[\delta_t\right]  +
\frac{2\alpha_t L}n  = \E\left[\delta_t\right] + \frac{2L\alpha_t}{n}\,.
\end{align}
}
Unraveling the recursion gives
\submit{
$\E\left[\delta_T\right] \le \frac{2L}{n}\sum_{t=1}^T\alpha_t\,.$
}{
\[
\E\left[\delta_T\right] \le \frac{2L}{n}\sum_{t=1}^T\alpha_t\,.
\]
} 
Plugging this back into equation~\eqref{eq:convex-diff}, 
\submit{gives the desired result.}{we obtain
\[
\E\left|f(w_T;z)-f(w_T';z)\right| \le \frac{2L^2}{n}\sum_{t=1}^T\alpha_t\,.
\]
Since this bounds holds for all $S,S'$ and $z,$ we obtain the desired bound on
the uniform stability.} 
\end{proof}

\submit{We refer the reader to the supplementary materials 
for our results on strongly convex optimization.}{
\subsection{Strongly Convex Optimization}\label{sec:sc-sgd}

In the strongly convex case we can bound stability with no dependence on the
number of steps at all. Assume that the function $f(w;z)$ is strongly convex
with respect to $w$ for all $z$.  Let  $\Omega$ be a compact, convex set over
which we wish to optimize.  Assume further that we can readily compute the
Euclidean projection onto the set $\Omega$, namely,
$\Pi_\Omega(v) = \arg\min_{w\in \Omega} \|w-v\|$.
In this section we restrict our attention to the projected stochastic
gradient method
\begin{equation}\label{eq:proj-sgd}
	w_{t+1} = \Pi_\Omega(w_t - \alpha_t  \nabla f(w_t;z_t))\,.
\end{equation}

A common application of the above iteration in machine learning is
solving Tikhonov regularization problems. Specifically, the empirical risk
is augmented with an additional regularization term,
\begin{equation}\label{eq:regularized-erm-problem}
		\mathrm{minimize}_w \; R_{S,\mu}[w] :=\frac{1}{n} \sum_{i=1}^n
    f(w; z_i) + \frac{\mu}{2} \|w\|_2^2 \,,
\end{equation}
where $f$ is as before a pre-specified loss function. We can assume without
loss of generality that $f(0;\cdot)=1$. Then, the optimal solution
of~\eqref{eq:regularized-erm-problem} must lie in a ball of radius $r$ about
$0$ where $r= \sqrt{2/\mu}\,.$
This fact can be ascertained by plugging in $w=0$ and noting that the
minimizer of \eqref{eq:regularized-erm-problem} must have a smaller cost, thus
$ \frac \mu 2 \|w^\star\|^2 \leq R_{S,\mu}[w^\star] \leq R_{S,\mu}[0] = 1 .$
We can now define the set $\Omega$ to be the ball of radius $r$, in
which case the projection is a simple scaling operation. Througout the rest
of the section we replace $f(w;z)$ with its regularized form, namely,
\submit{
	$f(w;z) \mapsto f(w; z) + \frac{\mu}{2} \|w\|_2^2 \,,$
}{
\[
	f(w;z) \mapsto f(w; z) + \frac{\mu}{2} \|w\|_2^2 \,,
\]
} 
which is strongly convex with parameter $\mu$.
Similarly, we will overload the constant $L$ to by setting
\begin{equation}\label{eq:M-def}
	L = \sup_{w\in \Omega} \sup_z \|\nabla f(w;z)\|_2\,.
\end{equation}
Note that if $f(w;z)$ is $\beta$-smooth for all $z$, then $L$ is always finite
as it is less than or equal to $\beta \operatorname{diam}(\Omega)$.  We need
to restrict the supremum to $w\in \Omega$ because strongly convex functions
have unbounded gradients on $\R^n$.  We can now state the first result about
strongly convex functions.

\begin{theorem}
\label{thm:sconvex}
Assume that the loss function $f(\cdot\,;z)$ is $\gamma$-strongly convex and $\beta$-smooth for all~$z.$ Suppose we run the projected SGM iteration~\eqref{eq:proj-sgd} with constant step size $\alpha\le 1/\beta$ for $T$ steps.
Then, SGM satisfies uniform stability with
\submit{
$\epsilon_{\mathrm{stab}} \le \frac{2 L^2}{\gamma n}\,.$
}{
\[
\epsilon_{\mathrm{stab}} \le \frac{2 L^2}{\gamma n}\,.
\]
} 
\end{theorem}

\submit{}{
\begin{proof}
The proof is analogous to that of Theorem~\ref{thm:convex} with a slightly different
recurrence relation. We repeat the argument for completeness.
Let $S$ and $S'$ be two samples of size $n$ differing in only a single example.
Consider the gradient updates $G_1,\dots,G_T$ and $G_1',\dots,G_T'$ induced by
running SGM on sample $S$ and $S',$ respectively. Let $w_T$ and $w_T'$ denote
the corresponding outputs of SGM.

Denoting $\delta_T=\|w_T-w_T'\|$ and appealing to the
boundedness of the gradient of $f,$ we have
\begin{equation}\label{eq:convex-diff2}
\E\left|f(w_T;z)-f(w_T';z)\right| \le
M\E\left[\delta_T\right]\,.
\end{equation}
%
Observe that at step $t,$ with probability $1-1/n,$ the example selected by
SGM is the same in both $S$ and $S'.$ In this case we have that $G_t=G_t'$.  At this stage, note that
\[
	\delta_t\leq \|w_{t-1}-\alpha \nabla f(w_t;z_t) -w_{t-1}'+\alpha \nabla f(w_t';z_t) \|
\]
because Euclidean projection does not increase the distance between projected points (see Lemma~\ref{lem:prox-expansive} below for a generalization of this fact).
We can now apply the following useful simplification of
Lemma~\ref{lem:sconvex2expansive} if $\alpha \leq 1/\beta$: since
$\frac{2\alpha\beta\gamma}{\beta+\gamma}\ge\alpha\gamma$ and $\alpha \gamma \le 1$,
$G_{f,\alpha}$ is $(1-\alpha\gamma)$-expansive. With probability $1/n$ the
selected example is different in which case we use that both $G_t$ and $G_t'$
are $\alpha{}M$-bounded as a consequence of Lemma~\ref{lem:Lipschitz2bounded}.
%
Hence, we can apply Lemma~\ref{lem:growth} and linearity of expectation to
conclude that for every $t,$
\begin{align}\label{eq:sconvex-recursion}
\E\delta_{t+1}
& \le \left(1-\frac1n\right)(1-\alpha \gamma)\E\delta_t + \frac1n(1-\alpha \gamma)\E\delta_t  +
\frac{2\alpha L}n  \\
\nonumber & = \left(1- \alpha \gamma \right)\E\delta_t    +
\frac{2\alpha L}n \,.
\end{align}
Unraveling the recursion gives
\[
\E\delta_T \le \frac{2L\alpha}{n}\sum_{t=0}^T\left(1- \alpha \gamma \right)^t
\le\frac{2L}{\gamma n}
\,.
\]
Plugging the above inequality into equation~\eqref{eq:convex-diff}, we obtain
\[
\E\left|f(w_T;z)-f(w_T';z)\right| \le \frac{2 L^2}{\gamma n}\,.
\]
Since this bounds holds for all $S,S'$ and $z,$ the lemma follows.
\end{proof}
} 

We would like to note that a nearly identical result holds for a ``staircase''
decaying step-size that is also popular in machine learning and stochastic
optimization.
\begin{theorem}
\label{thm:sconvex-decaying}
Assume that the loss function $f(\cdot\,;z)\in[0,1]$ is $\gamma$-strongly
convex has gradients bounded by $L$ as in~\eqref{eq:M-def}, and is $\beta$-smooth function for all $z.$ Suppose we
run SGM with step sizes $\alpha_t= \frac{1}{\gamma t}$.
Then, SGM has uniform stability of
\submit{
	$\epsilon_{\mathrm{stab}} \leq \frac{2L^2 + \beta \rho }{\gamma n}\,,$
}{
\[
	\epsilon_{\mathrm{stab}} \leq \frac{2L^2 + \beta \rho }{\gamma n}\,,
\]
}
where $\rho = \sup_{w \in \Omega} \sup_z f(w;z)$.
\end{theorem}
\submit{}{
\begin{proof}
Note that once $t>\frac{\beta}{\gamma}$, the iterates are contractive with
contractivity $1-\alpha_t \gamma \le 1 - \frac{1}{t}$.
Thus, for $t\geq t_0 :=  \frac{\beta}{\gamma}$, we have
\[
\begin{aligned}
	\E[\delta_{t+1}] &\leq (1-\tfrac{1}{n})(1- \alpha_t \gamma)\E[\delta_t] + \tfrac{1}{n}( (1-\alpha_t  \gamma)\E[\delta_t] + 2 \alpha_t L)\\
	&=(1- \alpha_t \gamma ) \E[\delta_t] + \frac{2 \alpha_t L }{n}\\
	& = \left(1 - \frac{1}{t}\right)\E[ \delta_t ] + \frac{2 L}{\gamma t n}\,.
\end{aligned}
\]
Assuming that $\delta_{t_0}=0$ and expanding this recursion, we find:
\[
	\E[\delta_T] \leq \sum_{t=t_0}^T \left\{\prod_{s=t+1}^T
		\left(1 - \frac{1}{s}\right) \right\} \frac{2  L }{\gamma t n}
	 = \sum_{t=t_0}^T \frac{t}{T} \frac{2 L }{\gamma t n} =
		\frac{T-t_0+1}{T}\cdot \frac{2  L }{\gamma n}\,.
\]
Now, the result follows from Lemma~\ref{lem:meta} with the fact that
$t_0=\frac{\beta}{\gamma}$.
\end{proof}
}

} 

\subsection{Non-convex optimization}

In this section we prove stability results for stochastic gradient methods that
do not require convexity. We will still assume that the objective function is
smooth and Lipschitz as defined previously.

The crux of the proof is to observe that SGM typically makes several steps
before it even encounters the one example on which two data sets in the
stability analysis differ.
\submit{}{
\begin{lemma}
\label{lem:meta}
Assume that the loss function $f(\cdot\,;z)$ is nonnegative and $L$-Lipschitz for all~$z.$
Let $S$ and $S'$ be two samples of size $n$ differing in only a single example.
Denote by $w_T$ and $w_T'$ the output of $T$ steps of SGM on $S$ and $S',$
respectively.  Then, for every $z\in Z$ and every
$t_0\in\{0,1,\dots,n\},$ under both the random update rule and the random
permutation rule, we have
\[
\E\left|f(w_T;z)-f(w_T';z)\right| \le \frac{t_0}n\sup_{w,z}f(w;z)
+ L\E\left[\delta_T\mid\delta_{t_0}=0\right]\,.
\]
\end{lemma}
} 
\submit{}{
\begin{proof}
Let $S$ and $S'$ be two samples of size $n$ differing in only a single example,
and let $z\in Z$ be an arbitrary example.
Consider running SGM on sample $S$ and $S'$, respectively.
As stated, $w_T$ and $w_T'$ denote the corresponding outputs of SGM.
Let ${\cal E}=\1{\delta_{t_0}=0}$ denote the event that $\delta_{t_0}=0.$
We have,
\begin{align*}
\E\left|f(w_T;z)-f(w_T';z)\right|
& =
\P\left\{ {\cal E}\right\} \E\left[\left|f(w_T;z)-f(w_T';z)\right| \mid {\cal
E}\right] \\
&\quad + \P\left\{ {\cal E}^c\right\} \E\left[\left|f(w_T;z)-f(w_T';z)\right| \mid {\cal
E}^c\right] \\
& \le
\E\left[\left|f(w_T;z)-f(w_T';z)\right| \mid {\cal E}\right]
+ \P\left\{ {\cal E}^c\right\}\cdot\sup_{w,z}f(w;z) \\
& \le
L\E\left[\left\|w_T-w_T'\right\| \mid {\cal E}\right]
+ \P\left\{ {\cal E}^c\right\}\cdot\sup_{w,z}f(w;z) \,.
\end{align*}
The second inequality follows from the Lipschitz assumption.

It remains to bound $\P\left\{ {\cal E}^c\right\}.$
Toward that end,
let $i^*\in\{1,\dots,n\}$ denote the position in which $S$ and $S'$ differ and
consider the random variable $I$ assuming the index of the first time step in
which SGM uses the example $z_{i^*}.$
Note that when $I>t_0,$ then we must have that $\delta_{t_0}=0,$ since the
execution on $S$ and $S'$ is identical until step $t_0.$ Hence,
\[
\P\left\{ {\cal E}^c\right\}
= \P\left\{ \delta_{t_0}\ne 0\right\}
\le \P\left\{ I \le t_0 \right\}\,.
\]
Under the random permutation rule, $I$ is a uniformly random number in
$\{1,\dots,n\}$ and therefore
\[
\P\left\{ I \le t_0 \right\}=\frac{t_0}n\,.
\]
This proves the claim we stated for the random permutation rule.
For the random selection rule, we have by the union bound
$
\P\left\{ I \le t_0 \right\}
\le \sum_{t=1}^{t_0}\P\left\{I=t\right\} =  \frac{t_0}n\,.
$
This completes the proof.
\end{proof}
} 

\begin{theorem}\label{thm:nonconvex}
Assume that $f(\cdot;z)\in[0,1]$ is an $L$-Lipschitz and $\beta$-smooth loss
function for every~$z.$ Suppose that we run SGM for $T$ steps with
monotonically non-increasing step sizes $\alpha_t\le c/t.$ Then, SGM has
uniform stability with
\[
\epsilon_{\mathrm{stab}}
\le
\frac{1+ 1/\beta c}{n-1}(2cL^2)^{\frac1{\beta c+1}}T^{\frac{\beta c}{\beta c+1}}
\]
In particular, omitting constant factors that depend on $\beta,$ $c,$ and
$L,$ we get
\submit{
$\epsilon_{\mathrm{stab}}\lessapprox \frac{T^{1-1/(\beta c+1)}}{n} \,.$
}{
\[
\epsilon_{\mathrm{stab}}\lessapprox \frac{T^{1-1/(\beta c+1)}}{n} \,.
\]
} 
\end{theorem}

\submit{}{
\begin{proof}
Let $S$ and $S'$ be two samples of size $n$ differing in only a single example.
Consider the gradient updates $G_1,\dots,G_T$ and $G_1',\dots,G_T'$ induced by
running SGM on sample $S$ and $S',$ respectively. Let $w_T$ and $w_T'$ denote
the corresponding outputs of SGM.

By Lemma~\ref{lem:meta}, we have for every $t_0\in\{1,\dots,n\},$
\begin{equation}\label{eq:nonconvex-diff}
\E\left|f(w_T;z)-f(w_T';z)\right| \le \frac{t_0}n +
L\E\left[\delta_T\mid\delta_{t_0}=0\right]\,,
\end{equation}
where $\delta_t=\|w_t-w_t'\|.$
To simplify notation, let $\Delta_t = \E\left[\delta_t\mid\delta_{t_0}=0\right].$
We will bound $\Delta_t$ as function of $t_0$ and then minimize for $t_0.$

Toward this goal, observe that at step $t,$ with probability $1-1/n,$ the
example selected by SGM is the same in both $S$ and $S'.$ In this case we have
that $G_t=G_t'$ and we can use the $(1+\alpha_t\beta)$-expansivity of the
update rule $G_t$ which follows from our smoothness assumption via
Lemma~\ref{lem:smooth2expansive}. With probability $1/n$ the selected example
is different in which case we use that both $G_t$ and $G_t'$ are
$\alpha_tL$-bounded as a consequence of Lemma~\ref{lem:Lipschitz2bounded}.

Hence, we can apply Lemma~\ref{lem:growth} and linearity of expectation to
conclude that for every $t\ge t_0,$
\begin{align*}
\Delta_{t+1}
& \le \left(1-\frac1n\right)(1+\alpha_t\beta)\Delta_t + \frac1n\Delta_t  +
\frac{2\alpha_t L}n\\
& \le \left(\frac1n +  (1-1/n)(1+c\beta/t)\right)\Delta_t + \frac{2cL}{tn}\\
& = \left(1 + (1-1/n)\frac{c\beta}t\right)\Delta_t + \frac{2cL}{tn}\\
& \le \exp\left((1-1/n)\frac{c\beta}t\right)\Delta_t + \frac{2cL}{tn}\,.
\end{align*}
Here we used that $1+x\le\exp(x)$ for all $x.$

Using the fact that $\Delta_{t_0}=0,$ we can unwind this recurrence relation
from $T$ down to $t_0+1.$ This gives
\begin{align*}
\Delta_T &\leq \sum_{t=t_0+1}^T \left\{\prod_{k=t+1}^T
\exp\left((1-\tfrac{1}{n}) \tfrac{\beta c}{k}\right) \right\}  \frac{2c L}{t n}\\
&= \sum_{t=t_0+1}^T \exp\left((1-\tfrac{1}{n}) \beta c \sum_{k={t+1}}^T
\tfrac{1}{k} \right) \frac{2c L}{t n}\\
&\leq \sum_{t=t_0+1}^T \exp\left((1-\tfrac{1}{n}) \beta c \log(\tfrac{T}{t})
\right) \frac{2c L}{t n} \\
& =  \frac{2cL}{n} T^{\beta c (1-1/n)} \sum_{t=t_0+1}^T t^{-\beta c (1-1/n) - 1}\\
&\leq  \frac1{(1-1/n)\beta c} \frac{2cL}{n} \left(\frac{T}{t_0}\right)^{\beta c (1-1/n)} \\
&\le \frac{2L}{\beta(n-1)} \left(\frac{T}{t_0}\right)^{\beta c}\,,
\end{align*}
Plugging this bound into~\eqref{eq:nonconvex-diff},
we get
\[
\E\left|f(w_T;z)-f(w_T';)\right| \le \frac{t_0}n +
\frac{2L^2}{\beta(n-1)} \left(\frac{T}{t_0}\right)^{\beta c}\,.
\]
Letting $q=\beta c,$ the right hand side is approximately minimized when
\[
t_0 = \left(2cL^2\right)^{\frac1{q+1}} T^{\frac{q}{q+1}}\,.
\]
This setting gives us
\begin{align*}
\E\left|f(w_T;z)-f(w_T';z)\right|
& \le \frac{1+1/q}{n-1}\left(2cL^2\right)^{\frac{1}{q+1}}
T^{\frac{q}{q+1}}
=\frac{1 + 1/\beta c}{n-1}(2cL^2)^{\frac1{\beta c+1}}T^{\frac{\beta c}{\beta c+1}}\,.
\end{align*}
Since the bound we just derived holds for all $S,S'$ and $z,$ we immediately get
the claimed upper bound on the uniform stability.
\end{proof}
} 

\section{Stability-inducing operations}\label{sec:operations}
In light of our results, it makes sense to analyse for operations that
increase the stability of the stochastic gradient method. We show in this
section that pleasingly several popular heuristics and methods indeed
improve the stability of SGM. Our rather straightforward analyses both
strengthen the bounds we previously obtained and help to provide an explanation
for the empirical success of these methods.

\paragraph{{Weight Decay and Regularization}.}
Weight decay is a simple and effective method that often improves
generalization~\cite{KroghHe92}.
\begin{definition}
Let $f\colon \Omega\to \Omega,$ be a differentiable function.
We define the \emph{gradient update with weight decay at rate~$\mu$} as
$G_{f,\mu,\alpha}(w)=(1-\alpha \mu)w-\alpha\nabla f(w).$
\end{definition}
It is easy to verify that the above update rule is equivalent to performing a
gradient update on the \emph{$\ell_2$-regularized} objective $g(w) = f(w) +
\frac{\mu}{2}\|w\|^2.$

\begin{lemma}
\label{lem:regularized2expansive}
Assume that $f$ is $\beta$-smooth. Then,
$G_{f,\mu,\alpha}$ is  $(1+\alpha(\beta-\mu))$-expansive.
\end{lemma}
\submit{}{
\begin{proof}
Let $G=G_{f,\mu,\alpha}.$ By triangle inequality and our smoothness assumption,
\begin{align*}
\|G(v)-G(w)\|
& \le (1-\alpha \mu)\|v-w\|+ \alpha \|\nabla f(w) - \nabla f(v)\|\\
& \le (1-\alpha \mu)\|v-w\| + \alpha\beta\|w-v\| \\
& = (1-\alpha \mu+\alpha\beta)\|v-w\|\,.
\end{align*}
\end{proof}
} 
The above lemma shows as that a regularization parameter $\mu$ counters a
smoothness parameter~$\beta.$ Once $r>\beta,$ the gradient update with
decay becomes contractive. Any theorem we proved in previous sections
that has a dependence on $\beta$ leads to a corresponding theorem for
stochastic gradient with weight decay in which $\beta$ is replaced with
$\beta-\mu.$

\paragraph{{Gradient Clipping}.}
It is common when training deep neural networks to enforce bounds on the
norm of the gradients encountered by SGD. This is often done by either
truncation, scaling, or dropping of examples that cause an exceptionally
large value of the gradient norm.  Any such heuristic directly leads to a
bound on the Lipschitz parameter~$L$ that appears in our bounds. It is also
easy to introduce a varying Lipschitz parameter $L_t$ to account for
possibly different values.

\paragraph{{Dropout}.}
Dropout~\cite{srivastava2014dropout} is a popular and effective heuristic for preventing
large neural networks from overfitting. Here we prove that, indeed, dropout
improves all of our stability bounds generically. From the point of view of
stochastic gradient descent, dropout is equivalent to setting a fraction of
the gradient weights to zero. That is, instead of updating with a stochastic
gradient $\nabla f(w;z)$ we instead update with a perturbed gradient
$D\nabla f(w;z)$ which is is typically identical to $\nabla f(w;z)$ in some
of the coordinates and equal to $0$ on the remaining coordinates, although
our definition is a fair bit more general.
\begin{definition}
We say that a randomized map $D\colon \Omega\to\Omega$ is a \emph{dropout
operator} with \emph{dropout rate}~$s$ if for every $v\in D$ we have
$\E\|Dv\|=s\|v\|.$
For a differentiable function $f\colon \Omega\to\Omega,$
we let $DG_{f,\alpha}$ denote the
\emph{dropout gradient update} defined as
$DG_{f,\alpha}(v)=v - \alpha D(\nabla f(v))$
\end{definition}
As expected, dropout improves the effective Lipschitz constant of
the objective function.
\begin{lemma}
\label{lem:dropout}
Assume that $f$ is $L$-Lipschitz. Then, the dropout update
$DG_{f,\alpha}$ with dropout rate $s$ is  $(s\alpha L)$-bounded.
\end{lemma}
\submit{}{
\begin{proof}
By our Lipschitz assumption and linearity of expectation,
\begin{align*}
\E\|G_{f,\alpha}(v)-v\|
= \alpha \E\| D\nabla f(v)\|
= \alpha s \E\|\nabla f(v)\|
\le \alpha s L,.
\end{align*}
\end{proof}
} 
\submit{}{From this lemma we can obtain various corollaries by replacing $L$ with $sL$
in our theorems.}

\submit{}{
\paragraph{{Projections and Proximal Steps}.} 
Related to
regularization, there are many popular updates which follow a stochastic
gradient update with a projection onto a set or some statistical shrinkage
operation.  The vast majority of these operations can be understood as
applying a proximal-point operation associated with a convex function.
Similar to the gradient operation, we can define the proximal update rule.
\begin{definition}
For a nonnegative step size $\alpha\ge 0$ and a function $f\colon \Omega\to\mathbb{R},$
we define the \emph{proximal update rule} $P_{f,\alpha}$ as
\submit{
$P_{f,\alpha}(w) = \arg \min_v \frac{1}{2}\|w-v\|^2 + \alpha  f(v)\,.$
}{
\begin{equation} \label{prox:eqn}
P_{f,\alpha}(w) = \arg \min_v \frac{1}{2}\|w-v\|^2 + \alpha  f(v)\,.
\end{equation}
}
\end{definition}
For example, Euclidean projection is the proximal point operation associated
with the indicator of the associated set. Soft-thresholding is the proximal
point operator associated with the $\ell_1$-norm. For more information, see
the surveys by Combettes and Wajs~\cite{Combettes05} or Parikh and
Boyd~\cite{parikh2013proximal}.

An elementary proof of the following Lemma, due to Rockafellar~\cite{Rockafellar76}, can be found in the appendix.
\begin{lemma}\label{lem:prox-expansive}
If $f$ is convex, the proximal update\submit{}{~(\ref{prox:eqn})} is $1$-expansive.
\end{lemma}

In particular, this Lemma implies that the Euclidean projection onto a convex
set is $1$-expansive. Note that in many important cases, proximal operators
are actually contractive. That is, they are $\eta$-expansive with $\eta<1$.
An notable example is when $f(\cdot)$ is the Euclidean norm for which the
update rule is $\eta$-expansive with $\eta = (1+\alpha)^{-1}$. So stability
can be induced by the choice of an appropriate prox-operation, which can
always be interpreted as some form of regularization.
} 

\submit{}{
\paragraph{{Model Averaging}.} 
Model averaging refers to the idea of
averaging out the iterates~$w_t$ obtained by a run of SGD. In convex
optimization, model averaging is sometimes observed to lead to better
empirical performance of SGM and closely replated updates such as the
Perceptron~\cite{FreundSc99}. Here we show that model averaging improves
our bound for the convex optimization by a constant factor.

\begin{theorem}\label{thm:convex-average}
Assume that $f\colon \Omega\to[0,1]$ is a decomposable convex $L$-Lipschitz
$\beta$-smooth function and that we run SGD with step sizes $\alpha_t\le
\alpha \le 2/\beta$ for $T$ steps.  Then, the average of the first $T$
iterates of SGD has uniform stability of
$
\epsilon_{\mathrm{stab}}
\le \frac{\alpha TL^2}{n}\,.
$
\end{theorem}
\submit{}{
\begin{proof}
Let $\bar{w}_T = \frac{1}{T} \displaystyle \sum_{t=1}^T w_t$ denoet the
average of the stochastic gradient iterates. Since
\[
w_t = \sum_{k=1}^t \alpha \nabla f(w_k; (x_k,y_k))\,,
\]
we have
\[
 \bar{w}_T =  \alpha \sum_{t=1}^T \frac{T-t+1}{T}  \nabla f(w_k; (x_k,y_k))
\]
Using Lemma 3.8, the deviation between $\bar{w}_t$ and $\bar{w}_t'$ obeys
\[
	\delta_t \leq (1-1/n) \delta_{t-1} +
		\frac{1}{n}\left( \delta_{t-1} + 2 \alpha L  \frac{T-t+1}{T}\right)\,.
\]
which implies
\[
	\delta_T \leq  \frac{2 \alpha L}{n} \sum_{t=1}^T \frac{T-t+1}{T} =
		\frac{\alpha L(T+1)}{n}\,.
\]
Since $f$ is $L$-Lipschitz, we have
\[
\E |f(\bar{w}_T)-f(\bar{w}_T')|
\le L\|\bar{w}_T-\bar{w}_T'\|
\le\frac{\alpha (T+1)L^2}{n}\,.
\]
Here the expectation is taken over the algorithm and hence the claim follows by
our definition of uniform stability.
\end{proof}
} 
} 

\section{Convex risk minimization}\label{sec:case-studies}

We now outline how our generalization bounds lead to bounds on the
population risk achieved by SGM in the convex setting.  We restrict our
attention to the convex case where we can contrast against known results.
The main feature of our results is that we show that one can achieve bounds
comparable or perhaps better than known results on stochastic gradient for
risk minimization by running for multiple passes over the data set.

The key to the analysis in this section is to decompose the risk estimates
into an \emph{optimization error} term and a stability term.  The
optimization error designates how closely we optimize the empirical risk or a
proxy of the empirical risk.  By optimizing with stochastic gradient, we
will be able to balance this optimization accuracy against how well we
generalize.  These results are inspired by the work of Bousquet and Bottou
who provided similar analyses for SGM based on uniform
convergence~\cite{Bottou08}. However, our stability results will yield
sharper bounds.

Throughout this section, our risk decomposition works as follows. We define
the \emph{optimization error} to be the gap between the empirical risk and
minimum empirical risk in expectation:
\submit{
	$\epsilon_{\mathrm{opt}}(w)\eqdef  \E\left [R_S[w] - R_S[w_\star^S]\right]$
where $w_\star^S = \arg\min_w R_S[w].$
}{
\[
	\epsilon_{\mathrm{opt}}(w)\eqdef  \E\left [R_S[w] - R_S[w_\star^S]\right]
	\,\mbox{ where }\, 
w_\star^S = \arg\min_w R_S[w]
	\,.
\]
}
By Theorem~\ref{thm:stab2gen}, the expected risk of a $w$ output by SGM is
bounded as
\submit{
	$\E[R[w]] \leq  \E[R_S[w]]  + \epsilon_{\mathrm{stab}}
	 \leq  \E[R_S[w_\star^S]]  + \epsilon_{\mathrm{opt}}(w) +  \epsilon_{\mathrm{stab}}.$
}{
\[
	\E[R[w]] \leq  \E[R_S[w]]  + \epsilon_{\mathrm{stab}}
	 \leq  \E[R_S[w_\star^S]]  + \epsilon_{\mathrm{opt}}(w) +  \epsilon_{\mathrm{stab}}\,.
\]
}
In general, the optimization error decreases with the number of SGM
iterations while the stability increases.  Balancing these two terms will
thus provide a reasonable excess risk against the empirical risk minimizer.
Note that our analysis involves the expected minimum empirical risk which
could be considerably smaller than the minimum risk. However, as we now
show, it can never be larger.

\begin{lemma}\label{lem:erm-beats-pop}
Let $w_\star$ denote the minimizer of the population risk and $w_\star^{S}$ denote the
minimizer of the empirical risk given a sampled data set $S$.
Then $\E[R_S[w_\star^{S}]] \leq R[w_\star]$.
\end{lemma}

\submit{}{
\begin{proof}
\begin{align*}
	R[w_\star] = \inf_w R[w]
	&=\inf_w \E_{z}[ f(w;z) ] \\
	&=\inf_w \E_S\left[ \tfrac{1}{n} \sum_{i=1}^n f(w;z_i) \right] \\
	&\geq\inf_w \E_S\left[ \tfrac{1}{n} \sum_{i=1}^n f(w_\star^{S};z_i)
	\right]\\
	&=\E_S\left[ \tfrac{1}{n} \sum_{i=1}^n f(w_\star^{S};z_i) \right]
= \E[R_S[w_\star^{S}]]\,.
\end{align*}
\end{proof}
} 

To analyze the optimization error, we will make use of a classical result due to Nemirovski and
Yudin~\cite{NY83}.
\begin{theorem}
\label{thm:NY}
Assume we run stochastic gradient descent with constant stepsize $\alpha$ on
a convex function
\submit{
$R[w] = \E_z[ f(w;z) ]\,.$
}{
\[
R[w] = \E_z[ f(w;z) ]\,.
\]
} 
Assume further that $\|\nabla f(w;z)\|\leq L$ and $\|w_0 - w_\star\|\leq D$
for some minimizer $w_\star$ of~$R$.  Let $\bar{w}_T$ denote the average of
the $T$ iterates of the algorithm.  Then we have
\submit{
	$R[\bar{w}_T] \le R[w_\star] + \tfrac{1}{2}\frac{D^2}{T\alpha} + \tfrac{1}{2} L^2 \alpha \,.$
}{
\[
	R[\bar{w}_T] \le R[w_\star] + \tfrac{1}{2}\frac{D^2}{T\alpha} + \tfrac{1}{2} L^2 \alpha \,.
\]
} 
\end{theorem}
The upper bound stated in the previous theorem is known to be tight even if the function
is $\beta$-smooth~\cite{NY83}

If we plug in the population risk for~$J$ in the previous theorem, we
directly obtain a generalization bound for SGM that holds when we make a
single pass over the data. The theorem requires fresh samples from the
distribution in each update step of SGM. Hence, given $n$ data points, we
cannot make more than $n$ steps, and each sample must not be used more than
once.

\begin{corollary}\label{cor:NY}
Let $f$ be a convex loss function satisfying
$\|\nabla f(w,z)\|\le L$ and let $w_\star$ be a minimizer of the population risk
$R[w]=\E_z f(w;z).$
Suppose we make a single pass of SGM over the sample
$S=(z_1,\dots,z_n)$ with a suitably chosen fixed step size starting from a
point $w_0$ that satisfies $\|w_0-w_\star\|\le D.$
Then, the average $\bar{w}_n$ of the iterates
satisfies
\submit{
	$\E[R[\bar{w}_n] ] \leq R[w_\star] + \frac{DL}{\sqrt{n}}\,.$
}{
\begin{equation}\label{eq:single-pass}
	\E[R[\bar{w}_n] ] \leq R[w_\star] + \frac{DL}{\sqrt{n}}\,.
\end{equation}
} 
\end{corollary}

We now contrast this bound with what follows from our results.
\begin{proposition}\label{prop:risk}
Let $S=(z_1,\dots,z_n)$ be a sample of size $n.$
Let $f$ be a $\beta$-smooth convex loss function satisfying
$\|\nabla f(w,z)\|\le L$ and let $w_\star^S$ be a minimizer of the empirical risk
$R_S[w]=\frac1n\sum_{i=1}^n f(w;z_i).$ Suppose we run $T$ steps of SGM with suitably chosen step
size from a starting point $w_0$ that satisfies $\|w_0-w_\star^S\|\le D.$
Then, the average $\bar{w}_T$ over the iterates satisfies
\submit{
$\E[R[\bar{w}_T]] \leq \E[R_S[w_\star^{S}]]
+ \frac{D L}{\sqrt{n}} \sqrt{\frac{n+2T}{T} }\,.$
}{
\[
\E[R[\bar{w}_T]] \leq \E[R_S[w_\star^{S}]]
+ \frac{D L}{\sqrt{n}} \sqrt{\frac{n+2T}{T} }\,.
\]
}
\end{proposition}

\submit{}{
\begin{proof}
On the one hand, applying Theorem~\ref{thm:NY} to the empirical risk~$R_S,$ we get
\[
	\epsilon_{\mathrm{opt}}(\bar{w}_T) \leq
		\tfrac{1}{2}\frac{D^2}{T\alpha} + \tfrac{1}{2} L^2 \alpha \,.
\]
Here, $w_\star^S$ is an empirical risk minimizer.
On the other hand, by our stability bound from Theorem~\ref{thm:convex-average},
\[
	\epsilon_{\mathrm{stab}} \leq  \frac{T L^2 \alpha}{n}
\]
Combining these two inequalities we have,
\[
	\E[R[\bar{w}_T]] \leq  \E[R_S[w_\star^{S}]] +
	\tfrac{1}{2}\frac{D^2}{T\alpha} + \tfrac{1}{2} L^2 \left(1  + \frac{2T}{n}\right)\alpha
\]
Choosing $\alpha$ to be
\[
\alpha = \frac{D \sqrt{n}}{L \sqrt{T (n + 2 T)}} ~,
\]
yields the bound provided in the proposition.
\end{proof}
} 

Note that the bound from our stability analysis is not directly comparable to
Corollary~\ref{cor:NY} as we are comparing against the expected minimum
empirical risk rather than the minimum risk.  Lemma~\ref{lem:erm-beats-pop}
implies that the excess risk in our bound is at most worse by a factor of
$\sqrt{3}$ compared with Corollary~\ref{cor:NY} when $T=n$.  Moreover, the
excess risk in our bound tends to a factor merely $\sqrt{2}$ larger than the
Nemirovski-Yudin bound as $T$ goes to infinity.  In contrast, the classical
bound does not apply when $T>n.$

\newcommand{\cifarpath}{plots/}
\newcommand{\mnistpath}{plots/}
\newcommand{\mnistsqpath}{plots/}
\newcommand{\ptbpath}{plots/}
\newcommand{\imagenetcpu}{plots/}
\newcommand{\imagenetsize}{plots/}

\section{Experimental Evaluation}

The goal of our experiments is to isolate the effect of training time, measured
in number of steps, on the stability of SGM. We evaluated broadly a
variety of neural network architectures and varying step sizes on a number of
different datasets.

To measure algorithmic stability we consider two proxies.  The first is the
Euclidean distance between the parameters of two identical models trained on
the datasets which differ by a single example.  In all of our proofs, we use
slow growth of this \emph{parameter distance} as a way to prove stability.
Note that it is not necessary for this parameter distance to grow slowly in
order for our models to be algorithmically stable. This is a strictly
stronger notion. Our second weaker proxy is to measure the generalization
error directly in terms of the absolute different between the test error and
training error of the model.

We analyzed four standard machine learning datasets each with their own
corresponding deep architecture.  We studied the LeNet architecture for MNIST,
the cuda-convnet architecture for CIFAR-10, the AlexNet model for ImageNet,
and the LSTM model for the Penn Treebank Language Model (PTB). Full details
of our architectures and training procedures can be found below.

In all cases, we ran the following experiment.  We choose a random example
from the training set and remove it.  The remaining examples constitute our set $S$. Then we create a set $S'$ by replacing a random element of $S$ with the element we deleted.  We train stochastic gradient descent with the same random seed on datasets $S$ and $S'$.  We record the Euclidean distance between the individual layers in the neural network after every $100$ SGM updates.  We also record the training and testing errors once per epoch.

To varying degrees, our experiments show four primary findings:
\begin{enumerate}
\item Typically, halving the step size roughly halves the generalization error.
This behavior is fairly consistent for both generalization error 
defined with respect to classification accuracy and cross entropy (the loss function used
for training). It thus suggests
that there is an intrinsic linear dependence on the step size in the
generalization error. The linear relationship between generalization error and step-size is quite pronounced in the Cifar10 experiments, as shown in Figure~\ref{fig:cifar10-alphas}.
\item We evaluate the Euclidean distance between the parameters of two models trained on two 
copies of the data differing in a random substitution. We observe that the parameter distance 
grows sub-linearly even in cases where our theory currently uses an exponential bound. This
shows that our bounds are pessimistic.
\item There is a close correspondence between the parameter distance and generalization
error. \emph{A priori}, it could have been the case that the generalization error is small even though
the parameter distance is large. Our experiments show that these two quantities
often move in tandem and seem to be closely related.
\item When measuring parameter distance it is indeed important that SGM does
not immediately encounter the random substitution, but only after some progress
in training has occurred. If we artificially place the corrupted data
point at the first step of SGM, the parameter distance can grow significantly
faster subsequently. This effect is most pronounced in the ImageNet experiments, as displayed in Figure~\ref{fig:early-vs-late}.
\end{enumerate}
We evaluated convolutional neural networks for image classification on three
datasets: MNIST, Cifar10 and ImageNet.  

\submit{}{

\subsection{Convolutional neural nets on Cifar}

}
Starting with Cifar10, we chose a standard model consisting of three
convolutional layers each followed by a pooling operation. 
This model roughly corresponds to that proposed by Krizhevsky \emph{et al.}~\cite{krizhevsky2012imagenet} and available in the
``cudaconvnet'' code\footnote{https://code.google.com/archive/p/cuda-convnet}.  
However, to make the experiments more
interpretable, we avoid all forms of regularization such as weight decay or
dropout. We also do not employ data augmentation even though this would greatly
improve the ultimate test accuracy of the model.  Additionally, we use only
constant step sizes in our experiments. With these restrictions the model we
use converges to below $20\%$ test error. While this is not state of the
art on Cifar10, our goal is not to optimize test accuracy but rather a simple,
interpretable experimental setup.
\submit{
\begin{figure}[t!]
\includegraphics[width=0.50\textwidth]{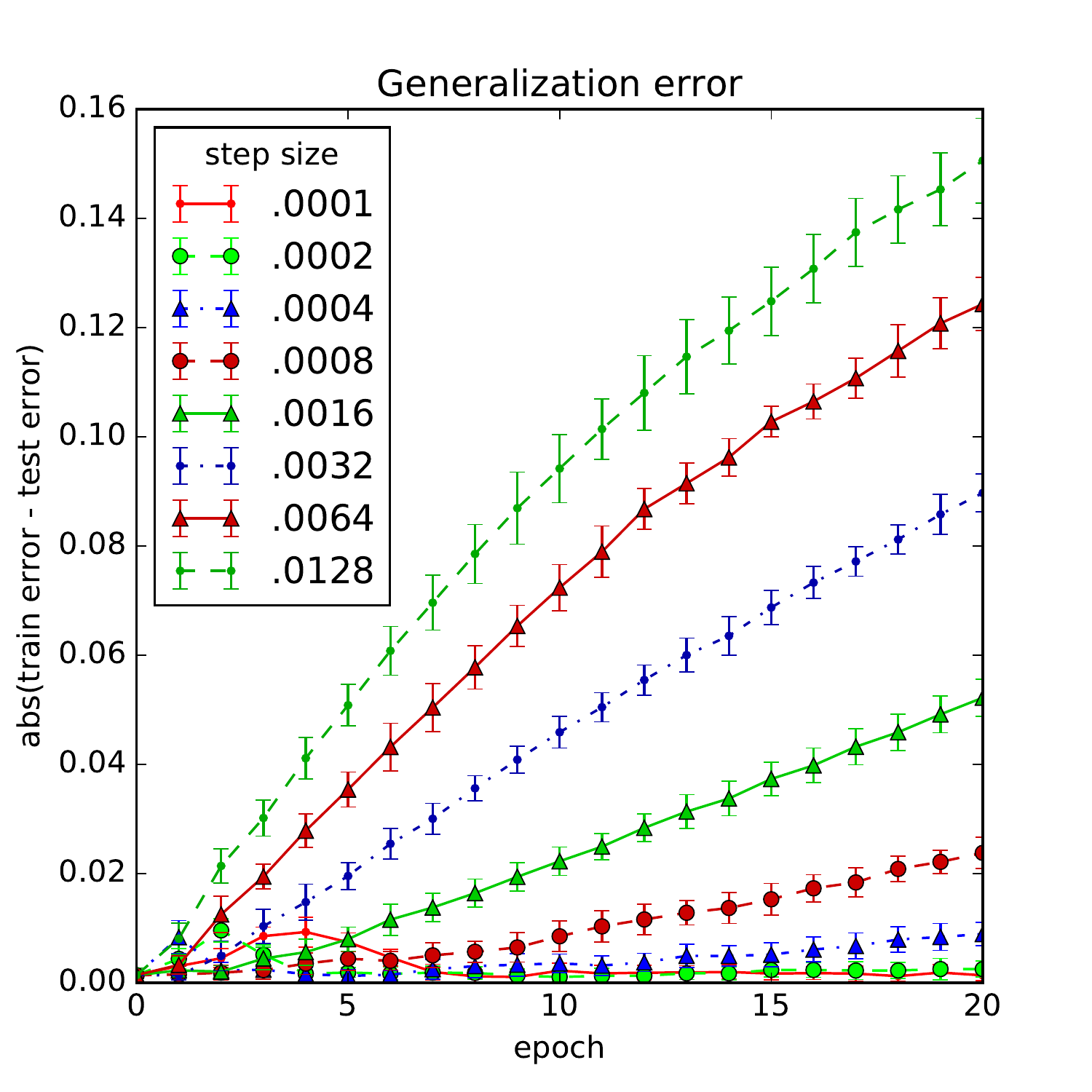}
\includegraphics[width=0.50\textwidth]{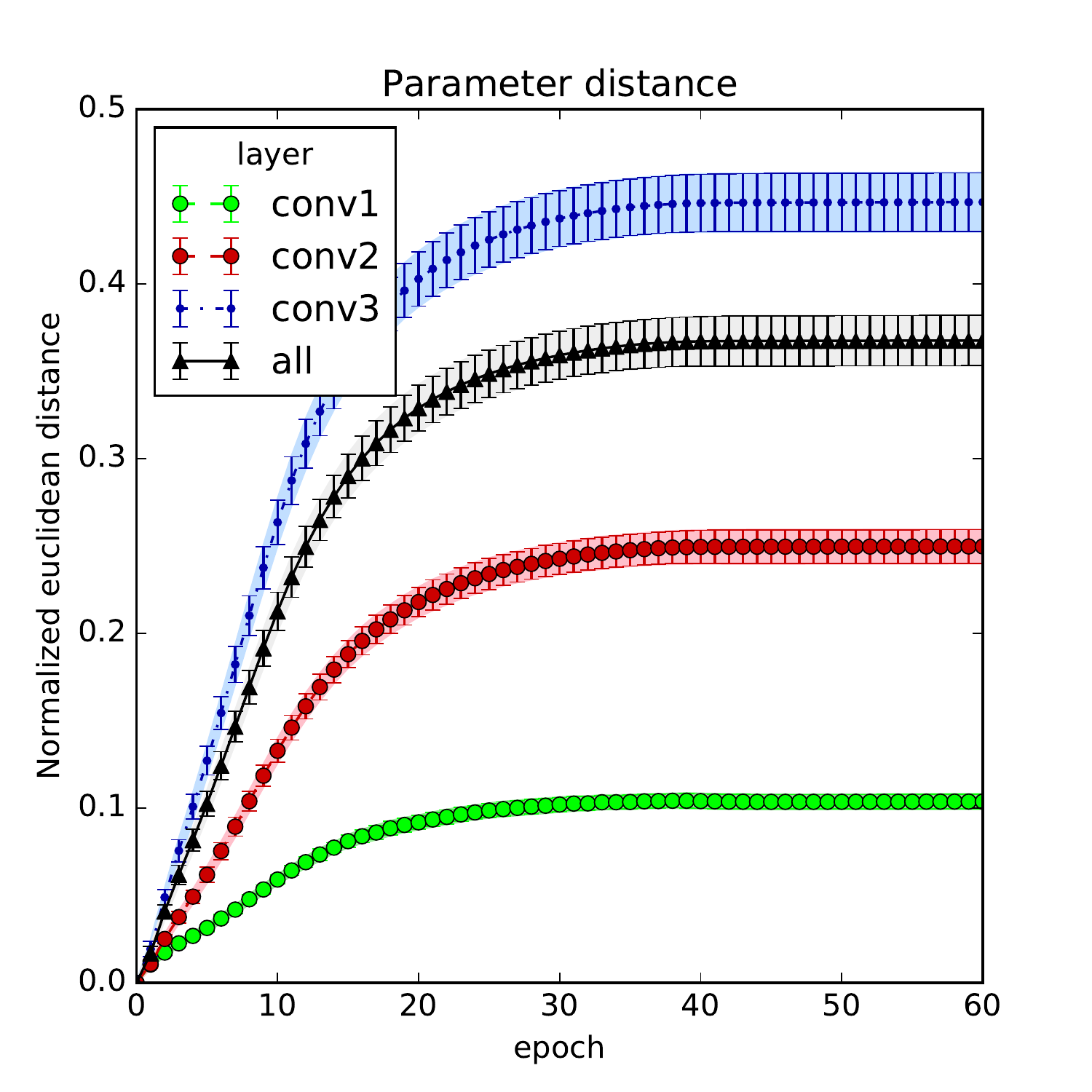}
\includegraphics[width=0.50\textwidth]{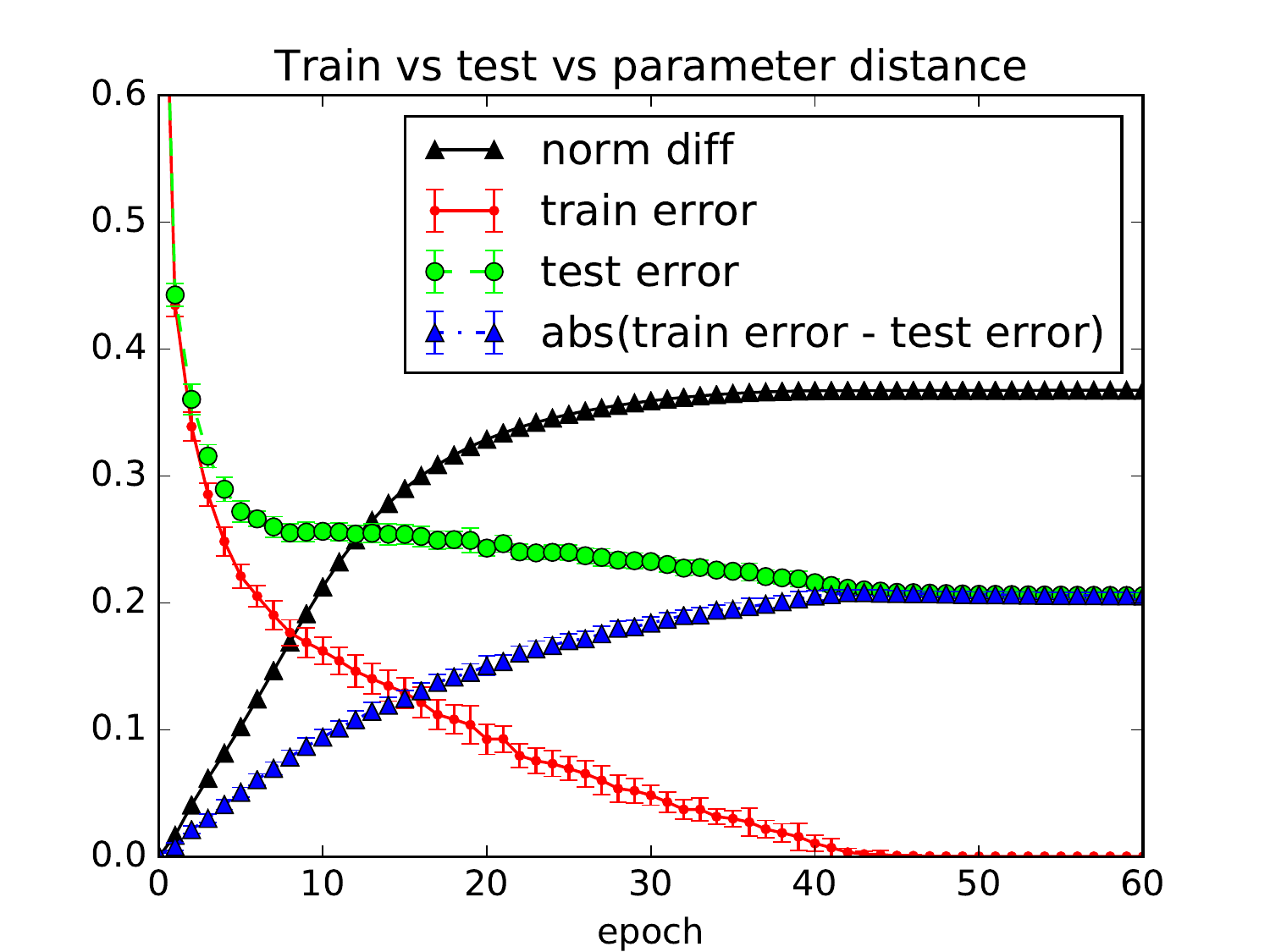}
\caption{Experimental results on Cifar. Top: Generalization error. Middle: Parameter divergence. Bottom: Comparison of train, test, and generalization eror with parameter divergence.}\label{fig:cifar10-alphas}

\end{figure}
}{
\begin{figure}
\includegraphics[width=0.50\textwidth]{{\cifarpath}/cifar10alphas-accuracy21epochs}
\includegraphics[width=0.50\textwidth]{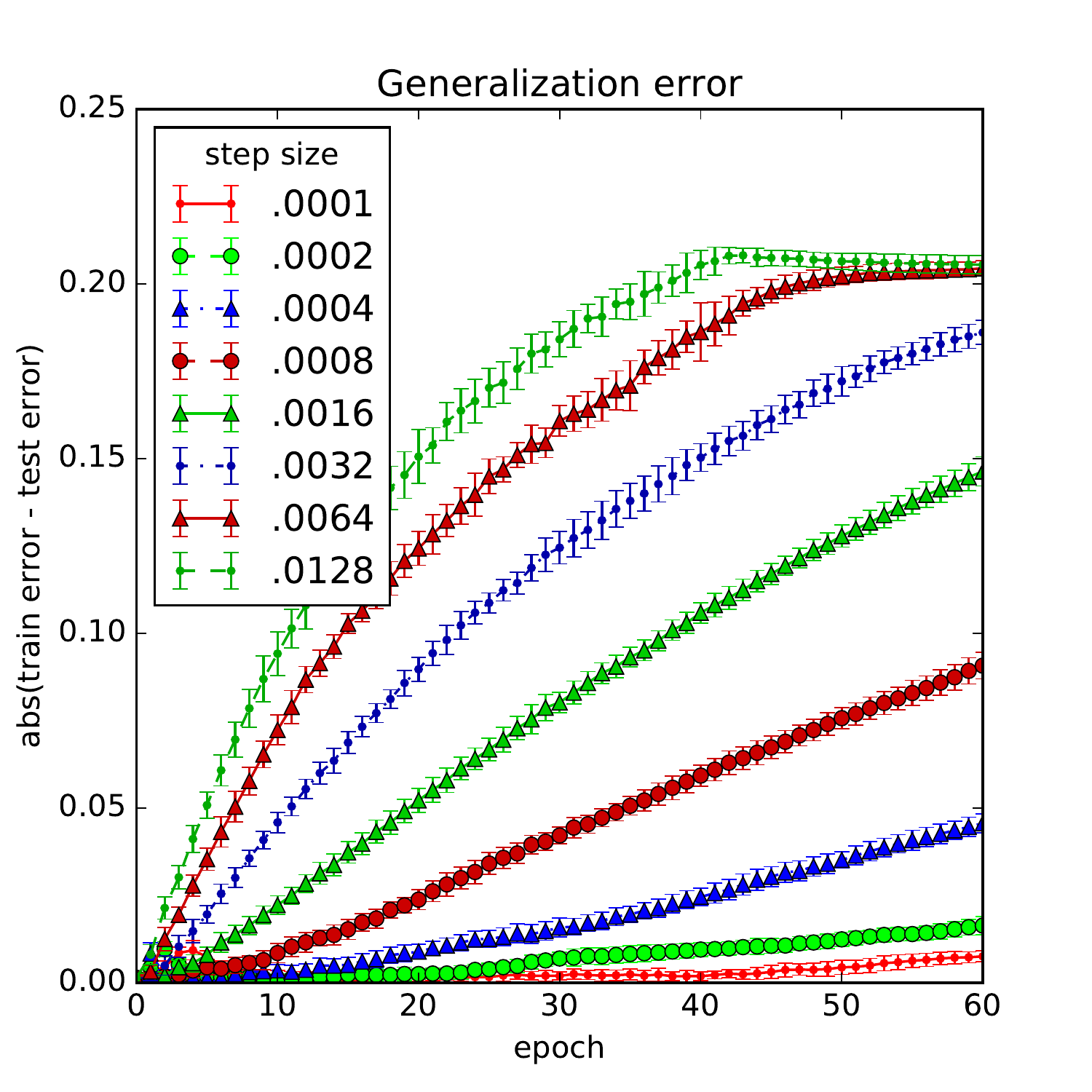}
\caption{Generalization error as a function of the number of epochs for varying step sizes on Cifar10.
Here generalization error is measured with respect to \emph{classification
accuracy}. Left: 20 epochs. Right: 60 epochs.}\label{fig:cifar10-alphas}
\end{figure}
\begin{figure}
\includegraphics[width=0.50\textwidth]{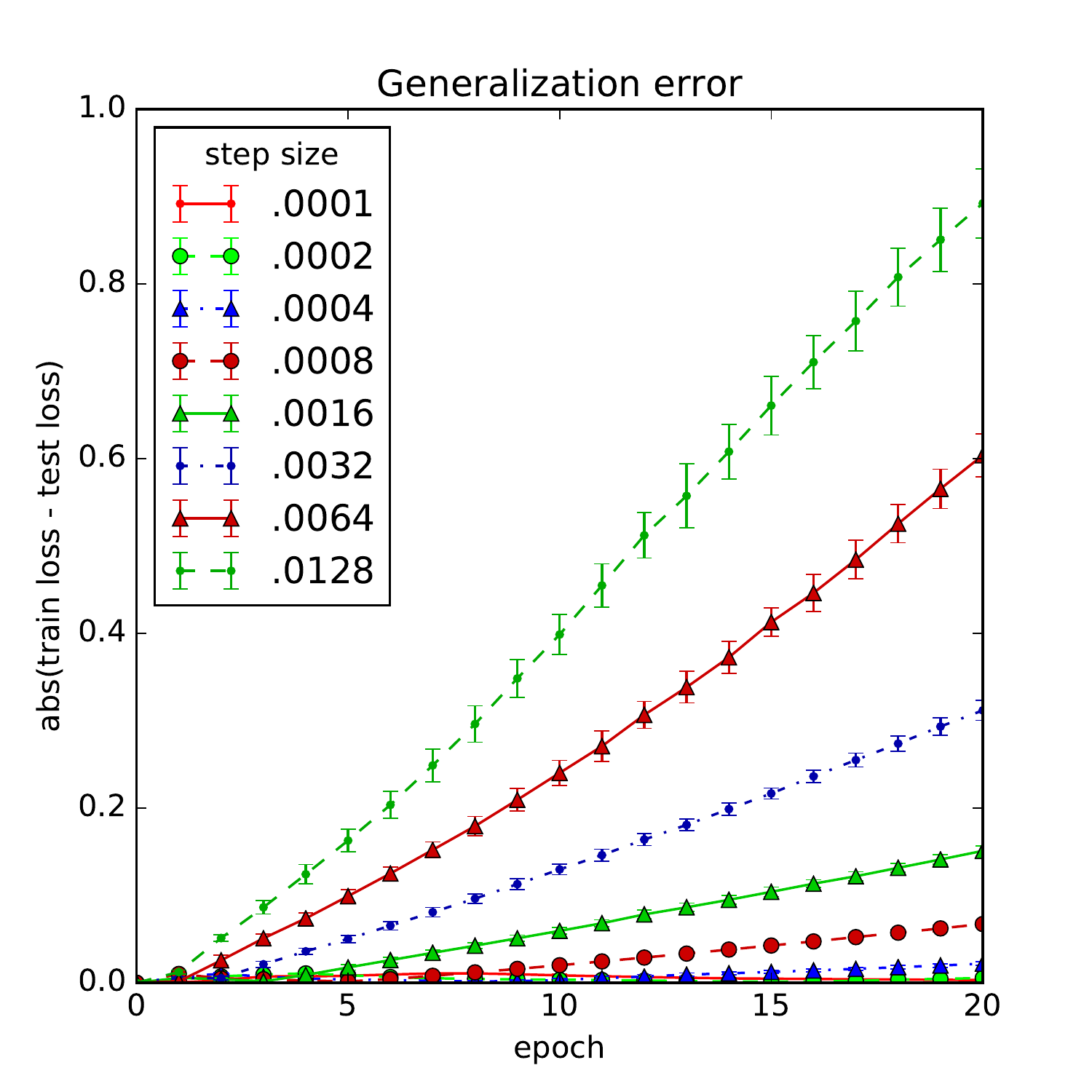}
\includegraphics[width=0.50\textwidth]{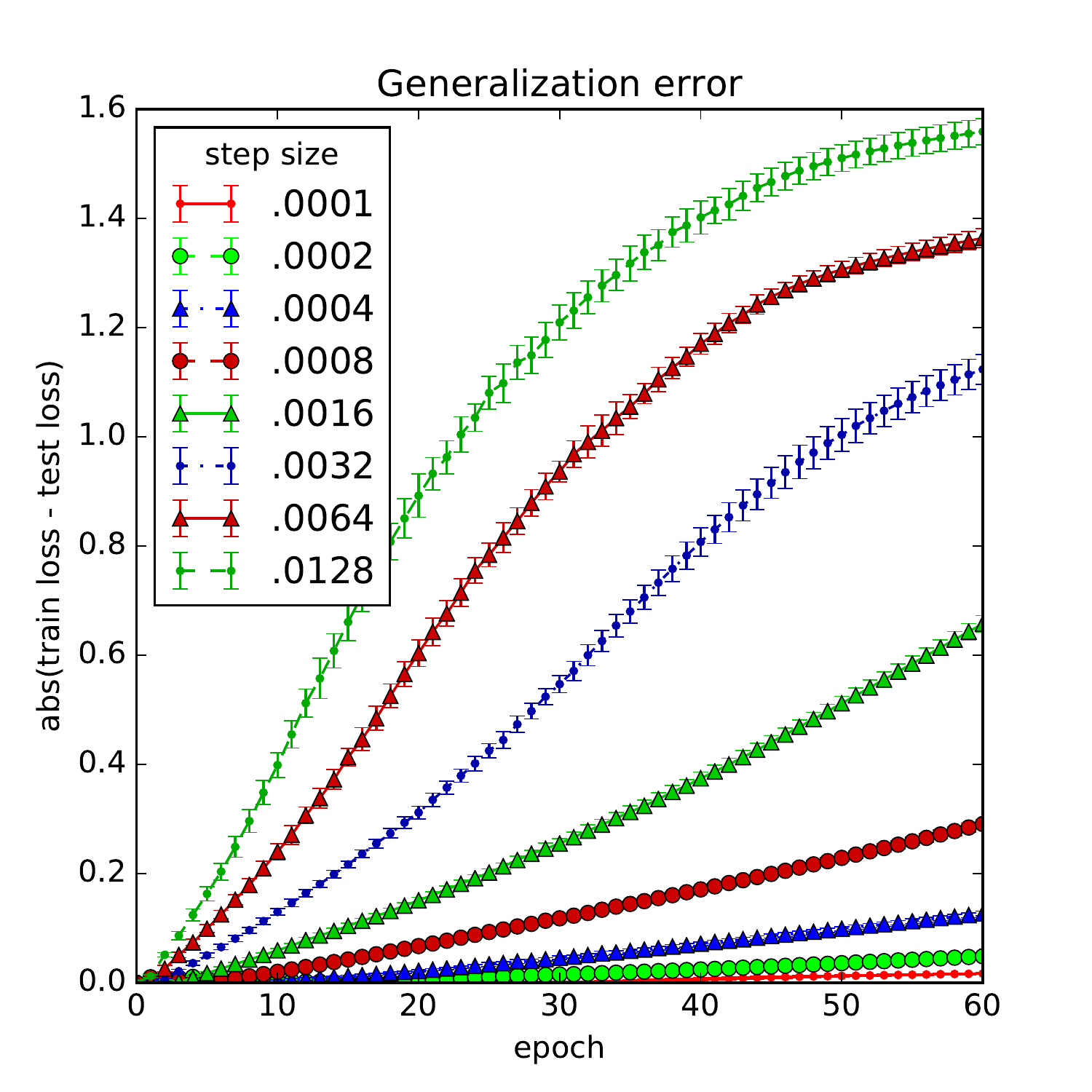}
\caption{Generalization error as a function of the number of epochs for varying
step sizes on Cifar10. Here, generalization error is measured with respect to \emph{cross
entropy} as a loss function. Left: 20 epochs. Right: 60 epochs.}
\end{figure}
\begin{figure}
\includegraphics[width=0.50\textwidth]{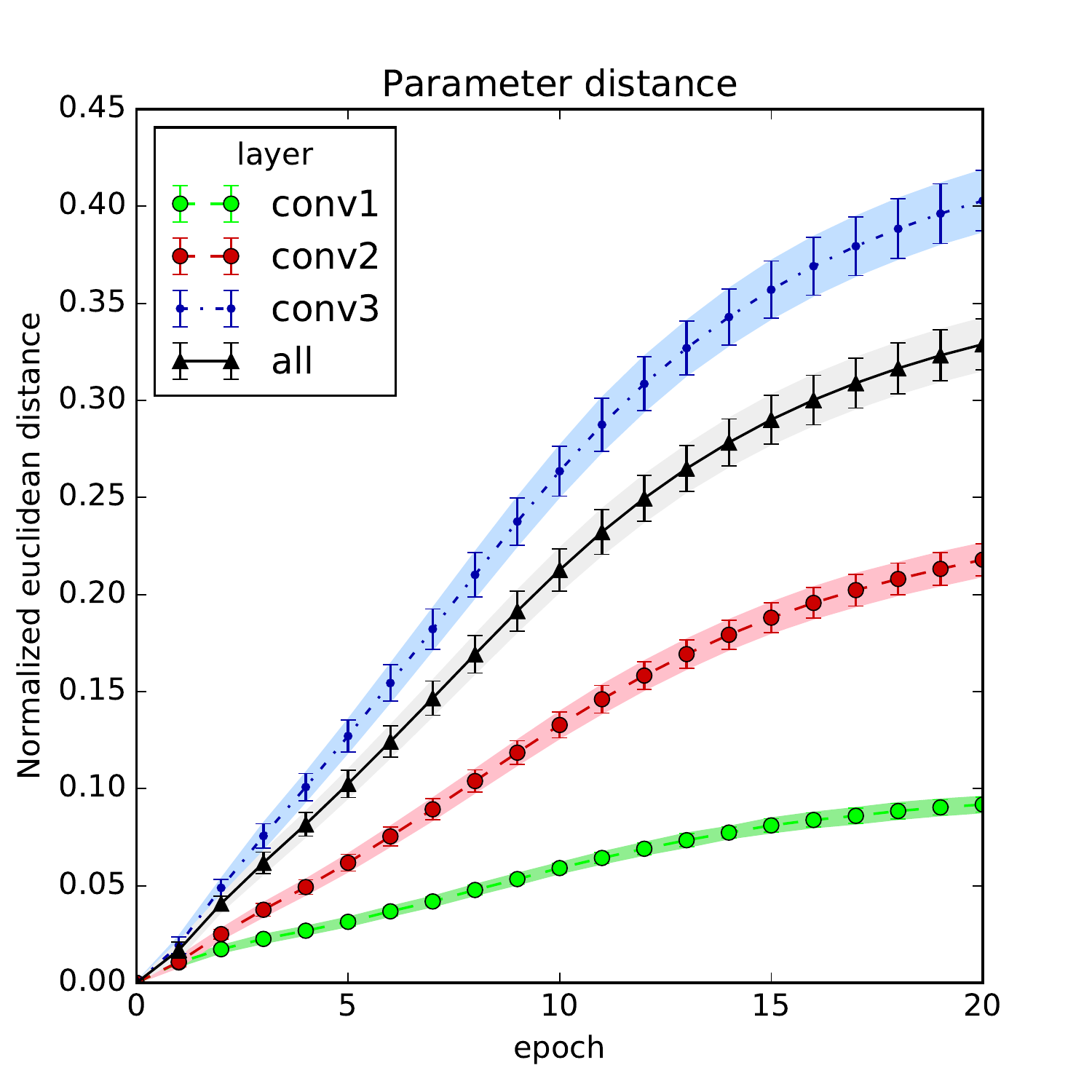}
\includegraphics[width=0.50\textwidth]{{\cifarpath}/cifar10norms-61epochs}
\caption{Normalized euclidean distance between parameters of two models trained under on different random substitution on Cifar 10. Here we show the differences between individual model layers.}
\end{figure}
\begin{figure}
\includegraphics[width=0.50\textwidth]{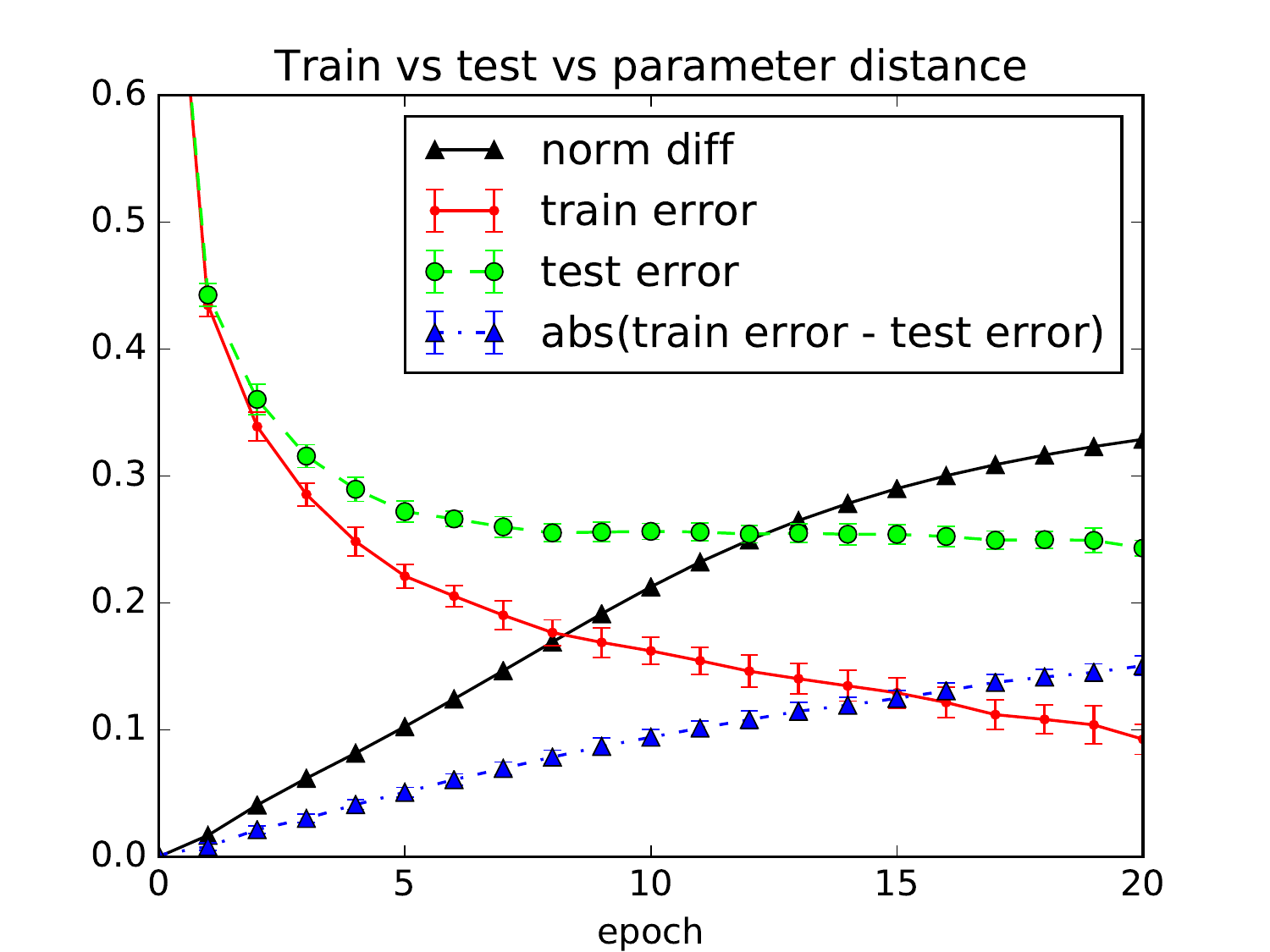}
\includegraphics[width=0.50\textwidth]{{\cifarpath}/cifar10combined-61epochs}
\caption{Parameter distance versus generalization error on Cifar10.}
\end{figure}
} 
\submit{}{
\subsection{Convolutional neural nets on MNIST}
}

The situation on MNIST is largely analogous to what we saw on Cifar10. We trained a LeNet inspired model with two convolutional layers and one fully-connected layer.  The first and second convolutional layers have 20 and 50 hidden units respectively.  This
model is much smaller and converges significantly faster than the Cifar10 models, typically achieving best test error in five epochs. We trained with minibatch size 60. As a result, the amount of overfitting is \submit{smaller.}{smaller as shown in Figure~\ref{fig:mnist}.}
\submit{}{
\begin{figure}
\includegraphics[width=0.50\textwidth]{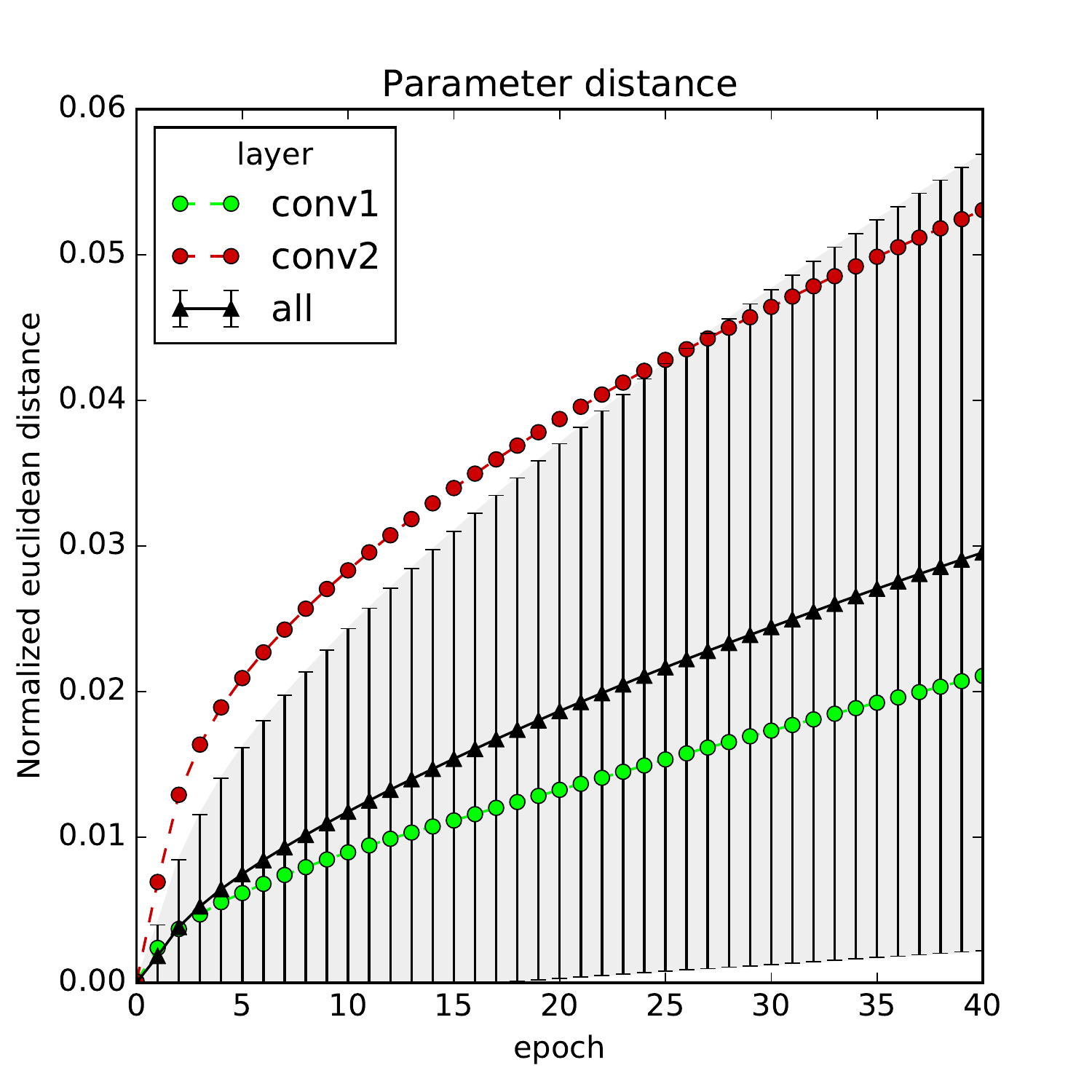}
\includegraphics[width=0.50\textwidth]{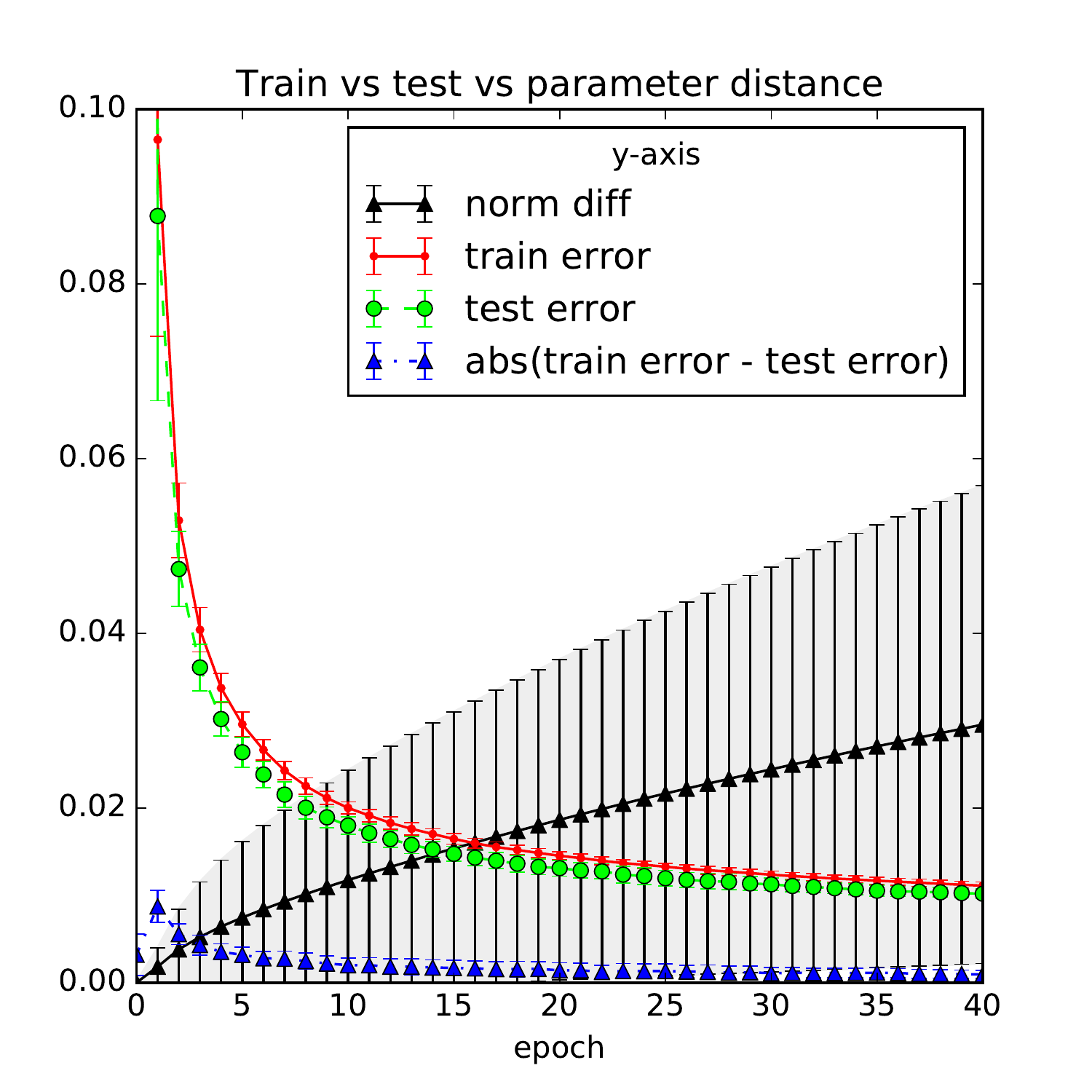}
\caption{Parameter distance and generalization error on MNIST.}
\label{fig:mnist}
\end{figure}
}

In the case of MNIST, we also repeated our experiments after replacing the
usual cross entropy objective with a squared loss objective.  
\submit{The results are in the supplementary materials.}{The results are displayed in Figure~\ref{fig:sq-loss}.}  It turned out that
this does not harm convergence at all, while leading to somewhat smaller
generalization error and parameter divergence.

\submit{}{
\begin{figure}
\includegraphics[width=0.50\textwidth]{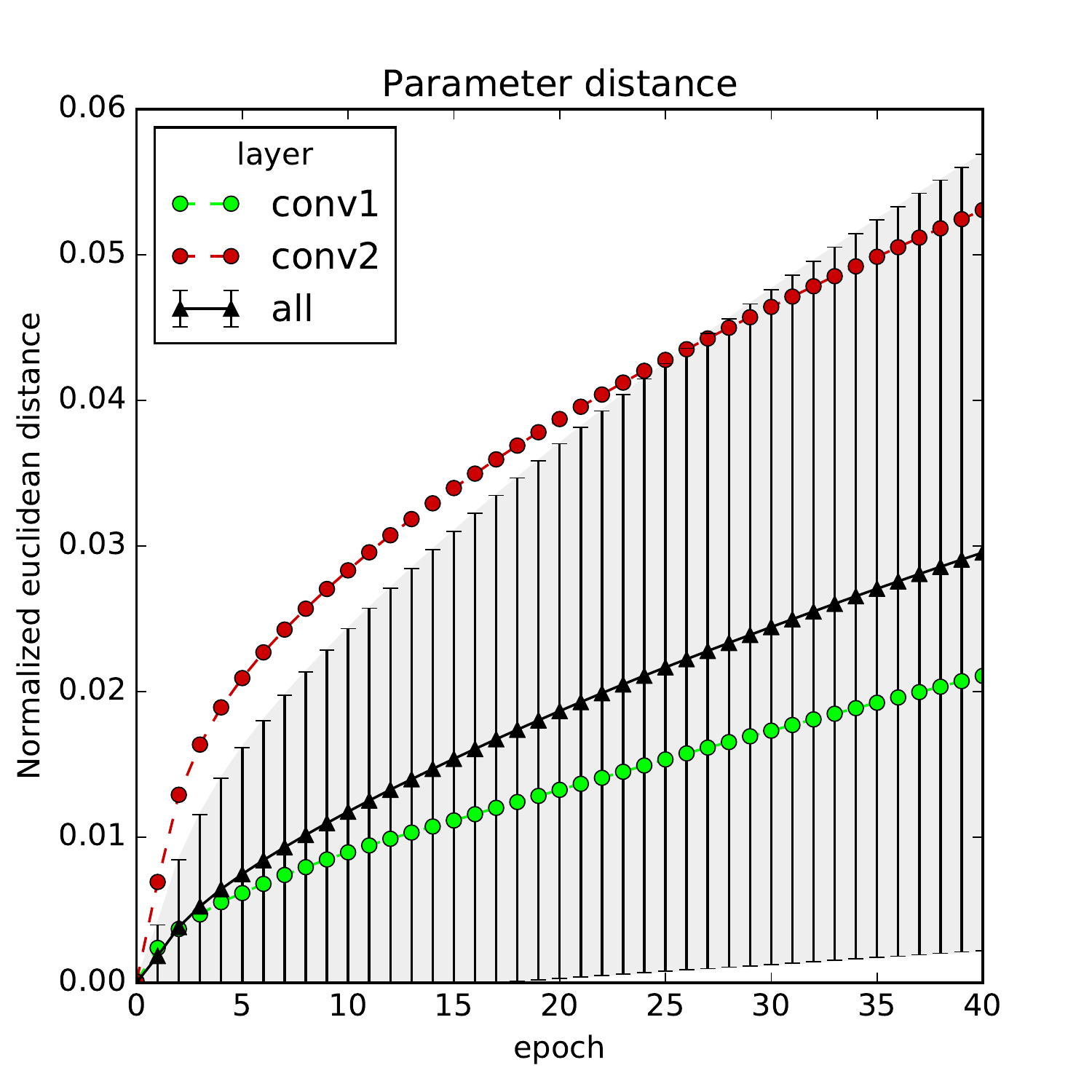}
\includegraphics[width=0.50\textwidth]{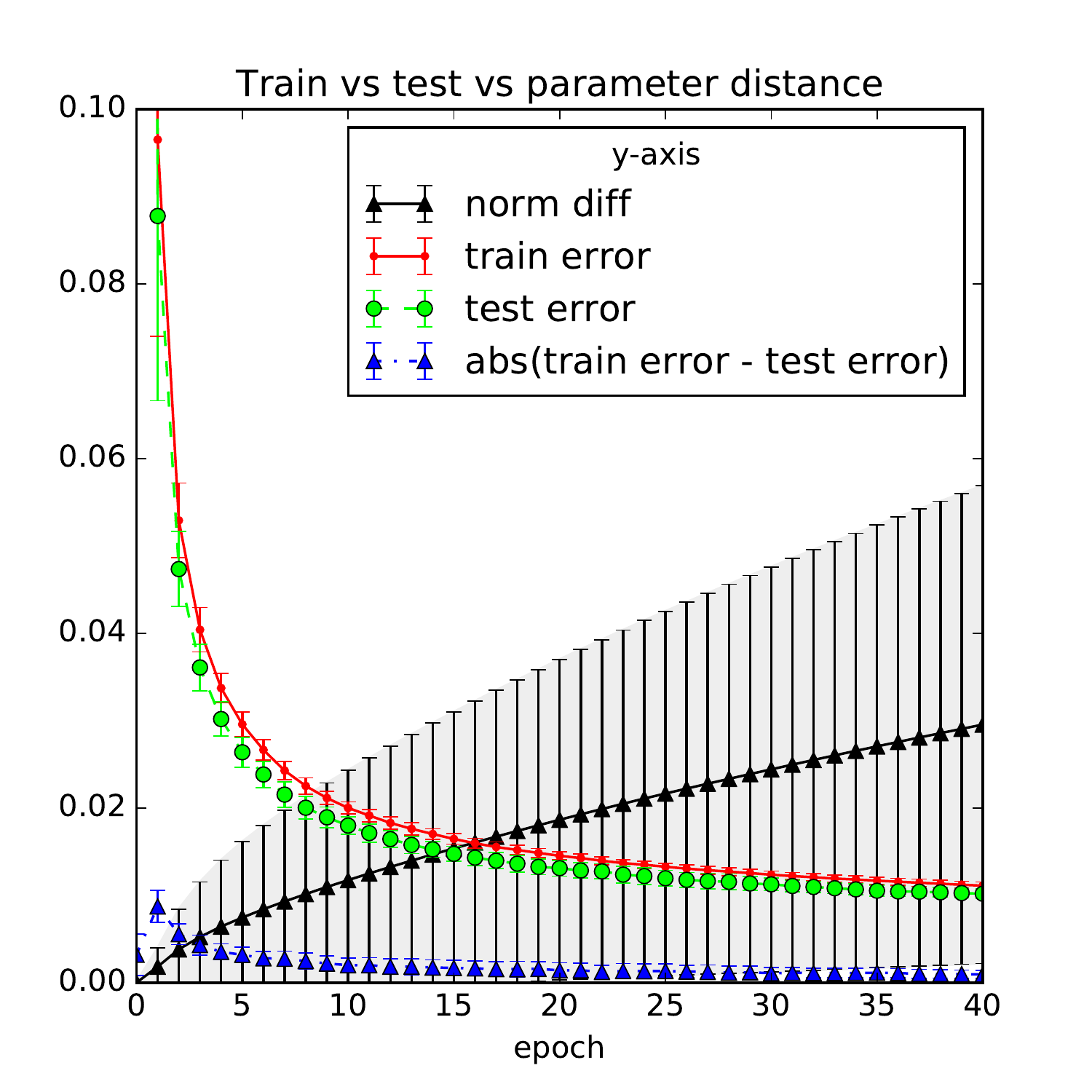}
\caption{Training on MNIST with squared loss objective instead of cross
entropy. Otherwise identical experiments as in the previous figure.}
\label{fig:sq-loss}
\end{figure}
} 

\subsection{Convolutional neural nets on ImageNet}

On ImageNet, we trained the standard AlexNet architecture~\cite{krizhevsky2012imagenet} using data augmentation, regularization, and dropout.  Unlike in the case of Cifar10, we were unable to find a setting of hyperparameters that yielded reasonable performance without using these techniques.   However, for Figure~\ref{fig:model-size}\submit{ (bottom)}{}, we did not use data-augmentation to exaggerate the effects of overfitting and demonstrate the impact scaling the model-size.  This figure demonstrates that the model-size appears to be a second-order effect with regards to generalization error, and step-size has a considerably stronger impact.
\submit{
\begin{figure}[t!]
\includegraphics[width=0.23\textwidth]{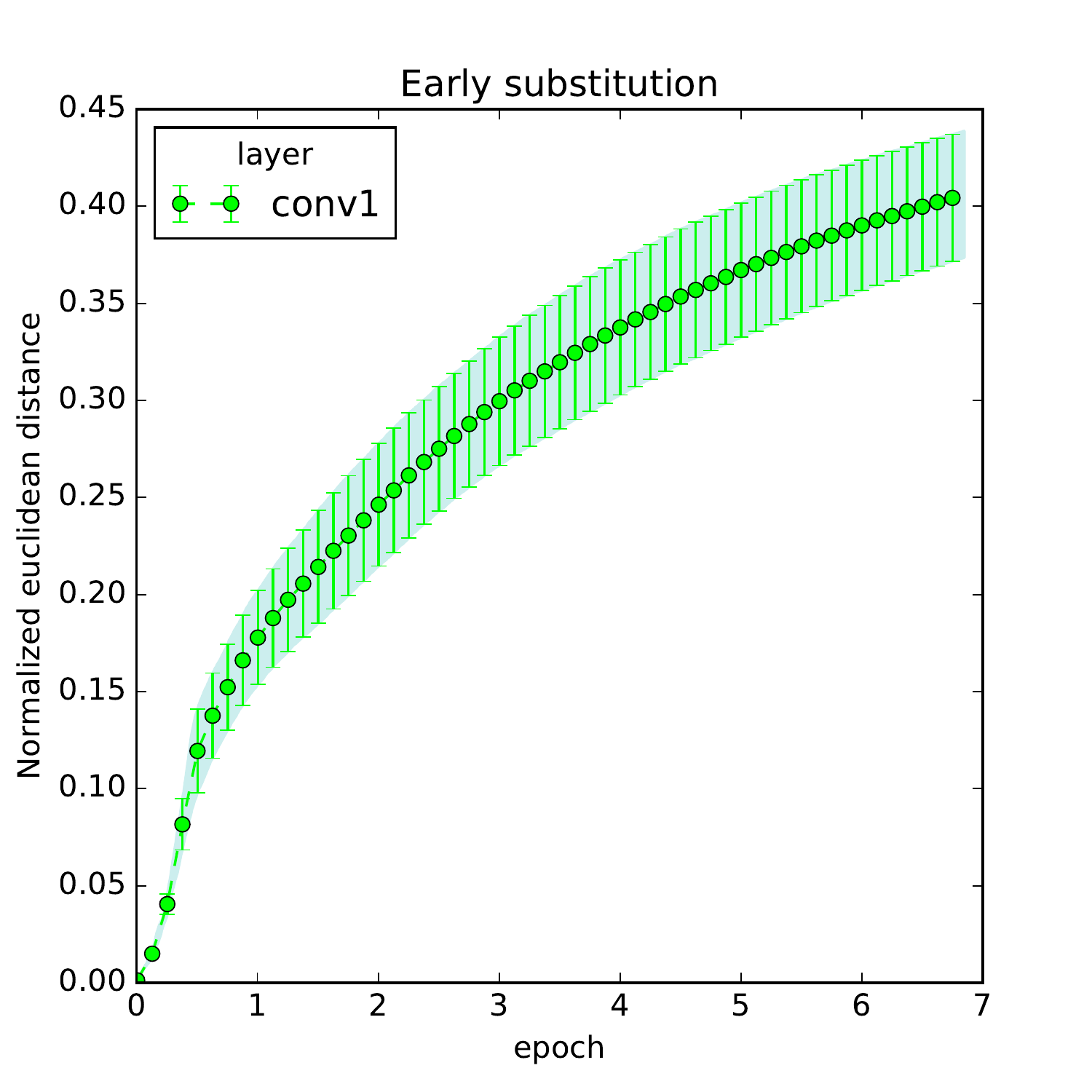}
\includegraphics[width=0.23\textwidth]{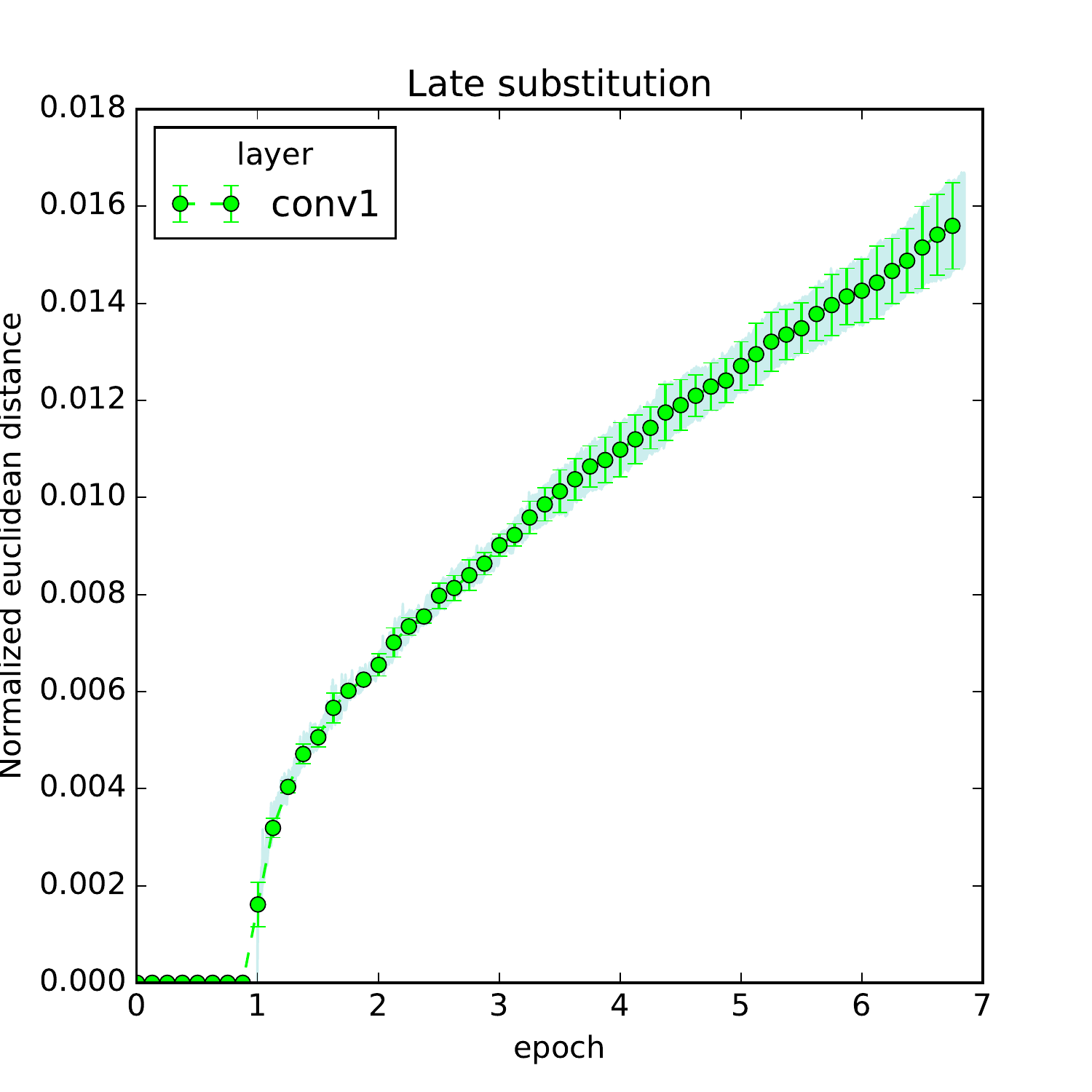}
\includegraphics[width=0.50\textwidth]{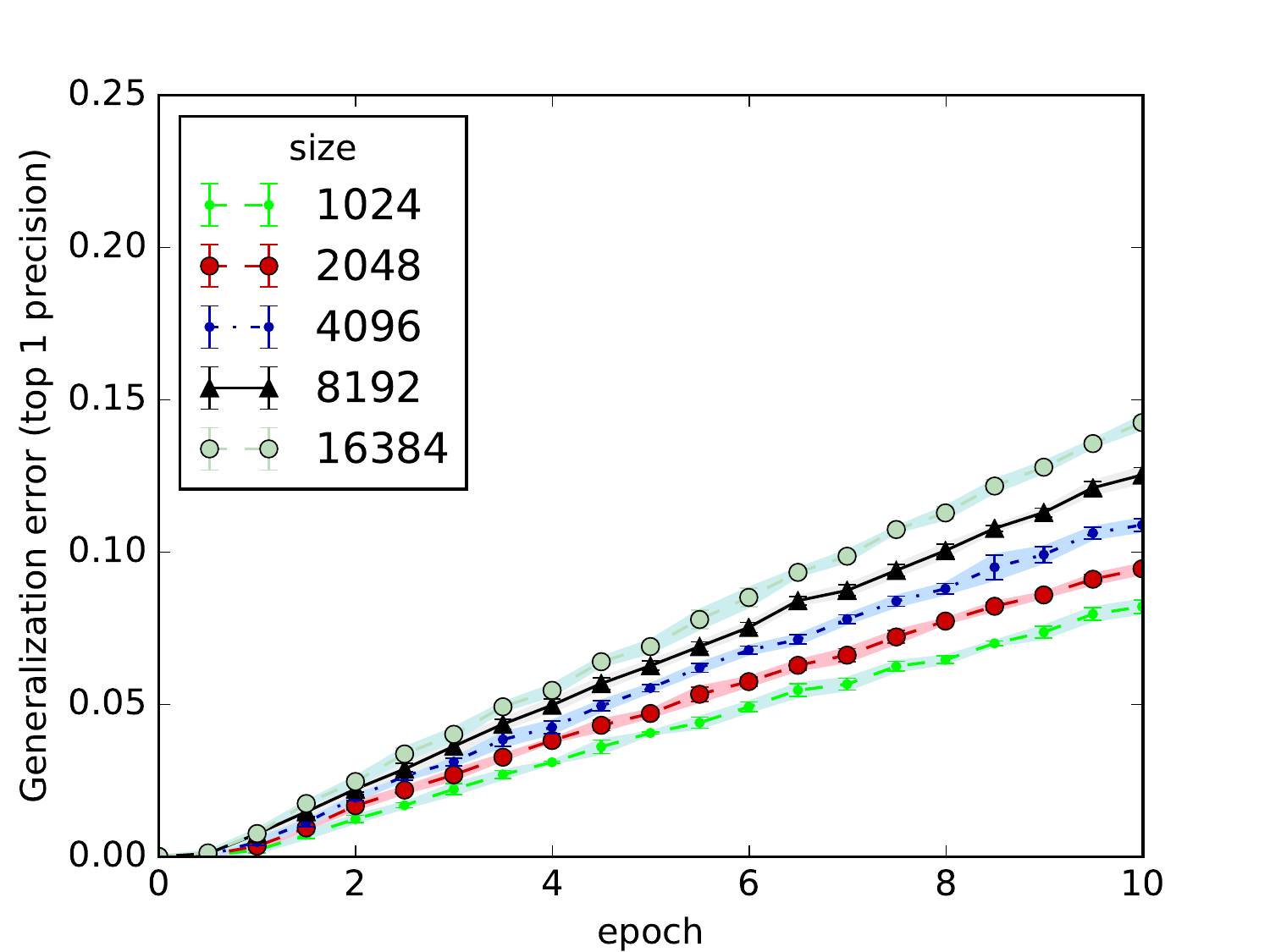}
\caption{Experiments on ImageNet. Top left: Parameter divergence with early substitution. Top right: Late substitution. Bottom: Generalization error for varying model size.}
\label{fig:imagenet}
\label{fig:early-vs-late}
\label{fig:model-size}
\end{figure}
}{
\begin{figure}
\includegraphics[width=0.50\textwidth]{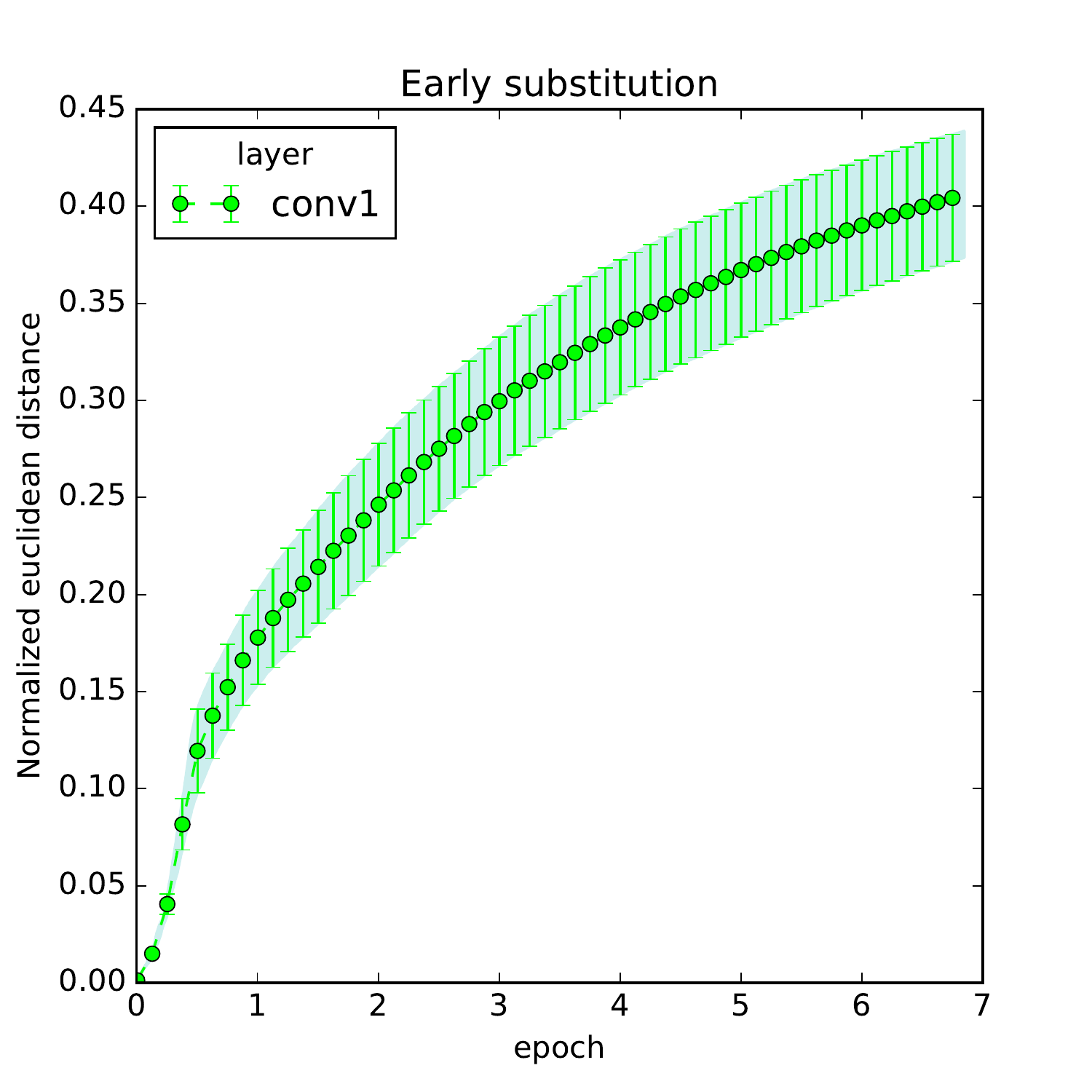}
\includegraphics[width=0.50\textwidth]{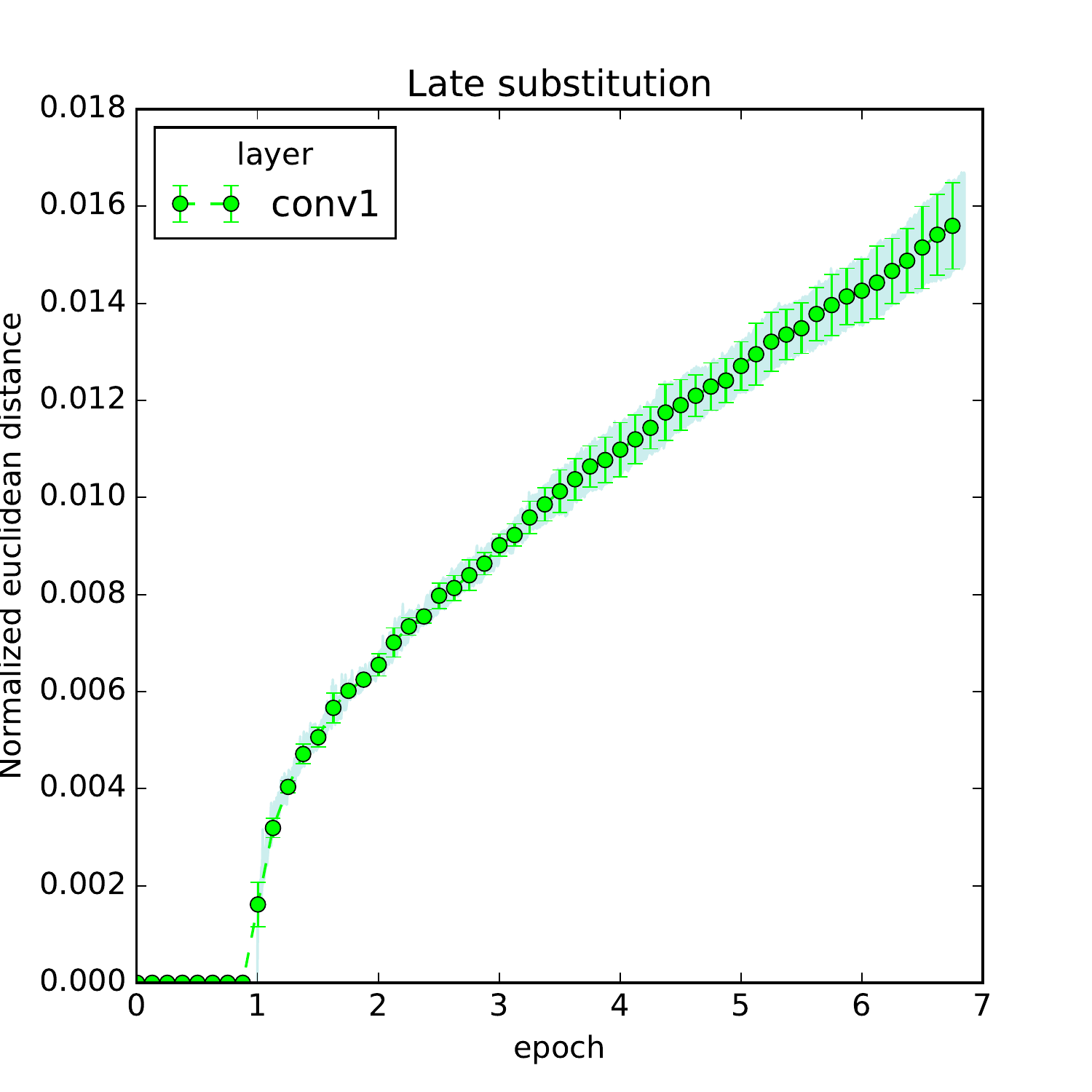}
\caption{\submit{Top:}{Left:} Performing a random substitution at the beginning of each epoch on AlexNet. \submit{Bottom:}{Right:} Random substitution at the end of each epoch. The parameter divergence is considerably smaller under late substitution.}
\label{fig:early-vs-late}
\end{figure}

\begin{figure}
\includegraphics[width=0.50\textwidth]{{\imagenetsize}/imagenet-size-top1}
\submit{}{\includegraphics[width=0.50\textwidth]{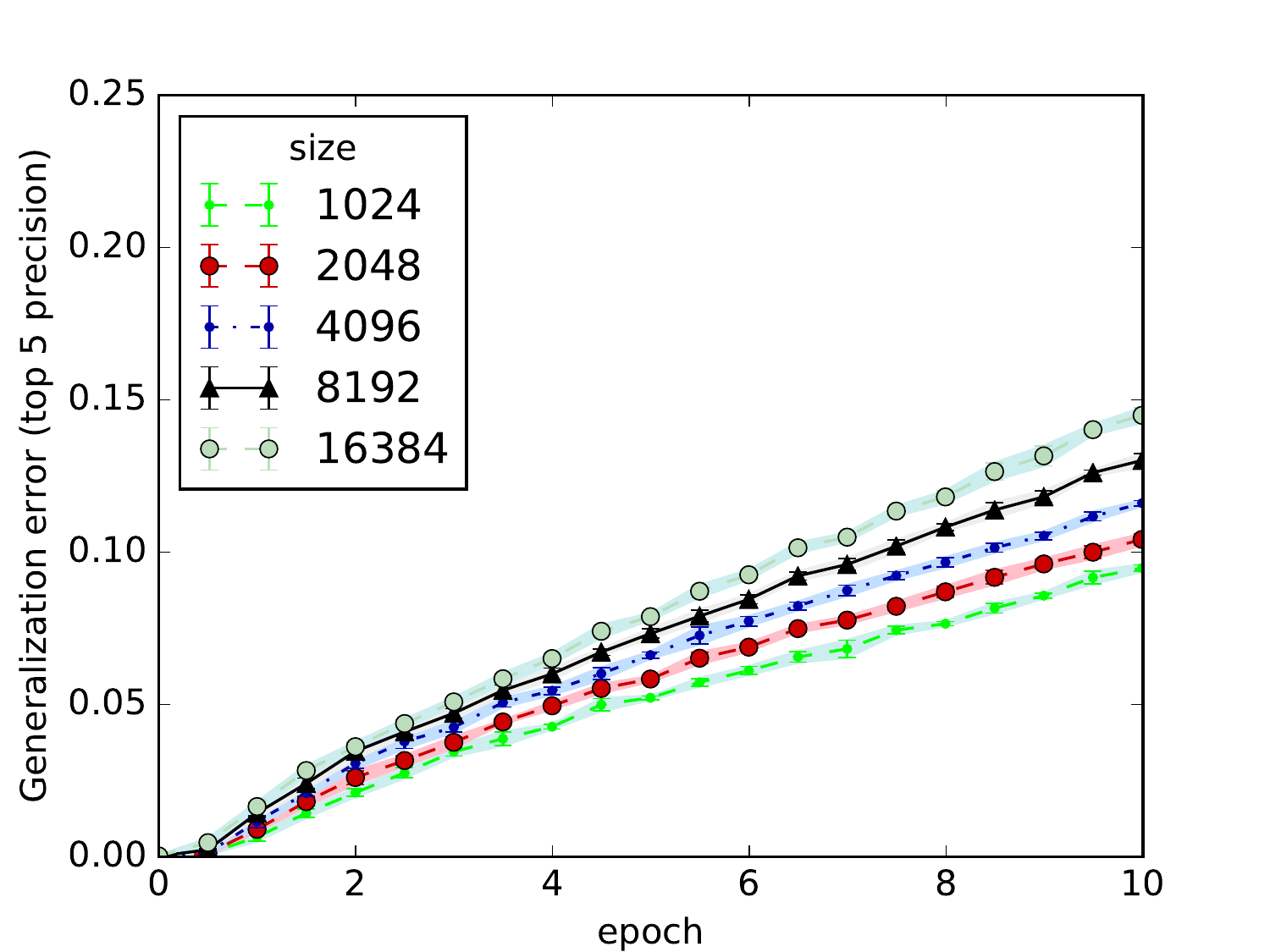}}
\vspace{-5mm}
\caption{\submit{Generalization error for varying model sizes on Imagenet}{Left: Generalization error in terms of top $1$ precision for varying model size on Imagenet.} \submit{}{Right: The same with top $5$ precision.}}
\label{fig:model-size}
\end{figure}
} 

\subsection{Recurrent neural networks with LSTM}
We also examined the stability of recurrent neural networks. Recurrent models
have a considerably different connectivity pattern than their convolutional
counterparts. Specifically, we looked at an LSTM architecture that was used
by Zaremba~\emph{et al.} for language modeling~\cite{zaremba2014recurrent}.
We focused on word-level prediction experiments using the Penn Tree Bank (PTB)
~\cite{marcus1993building}, consisting of 929,000 training words, 73,000
validation words, and 82,000 test words. PTB has 10,000 words in its
vocabulary\footnote{The data can be accessed at the
URL~\url{http://www.fit.vutbr.cz/~imikolov/rnnlm/simple-examples.tgz}}.
Following Zaremba~\emph{et al.}, we trained regularized LSTMs with two layers
that were unrolled for 20 steps. We initialize the hidden states to zero. We
trained with minibatch size 20. The LSTM has 200 units per layer and its
parameters are initialized to have mean zero and standard deviation of 0.1.
We did not use dropout to enhance reproducibility. Dropout would only increase
the stability of our models. The results are displayed in
Figure~\ref{fig:ptb}.
\submit{
\begin{figure}[h!]
\includegraphics[width=0.23\textwidth]{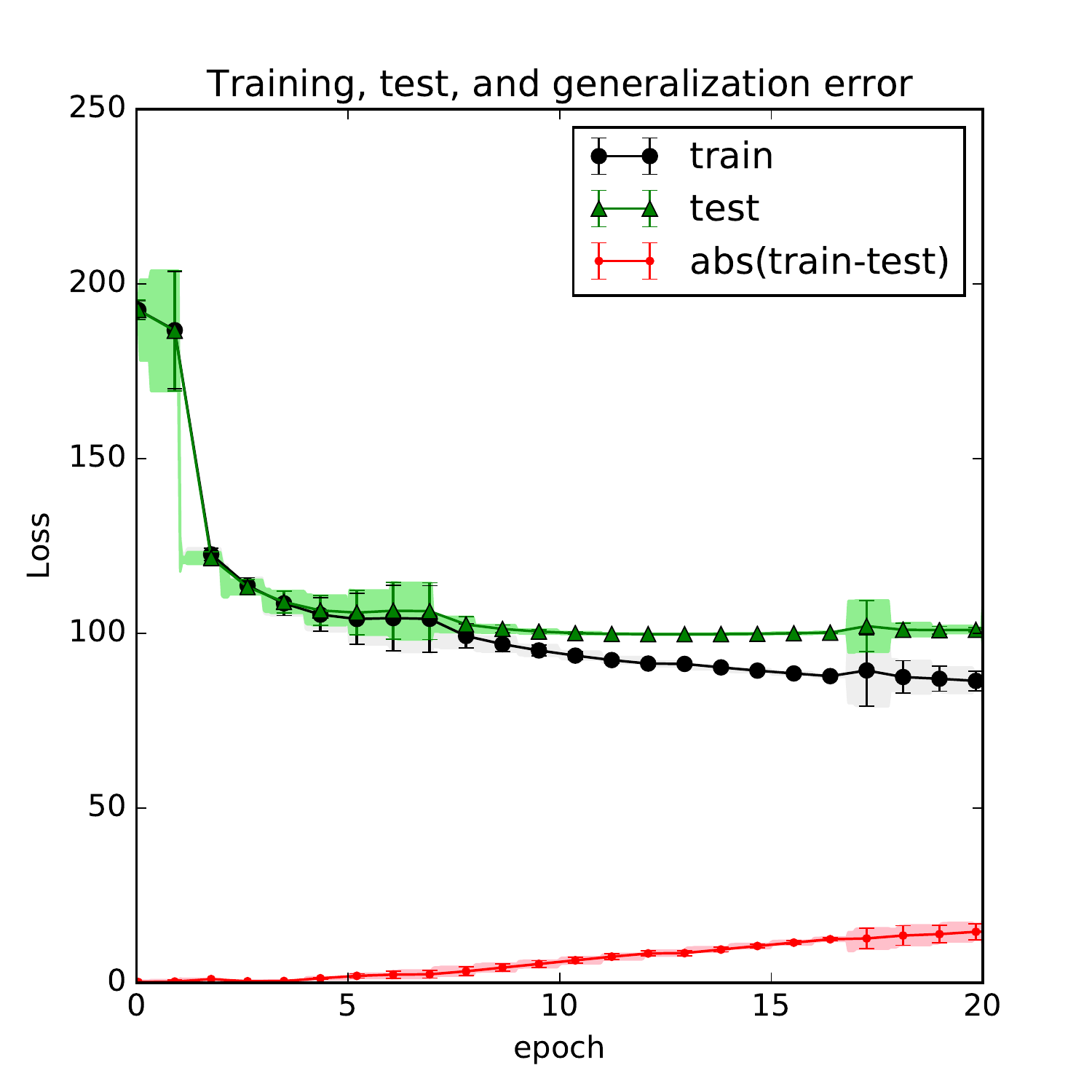}
\includegraphics[width=0.23\textwidth]{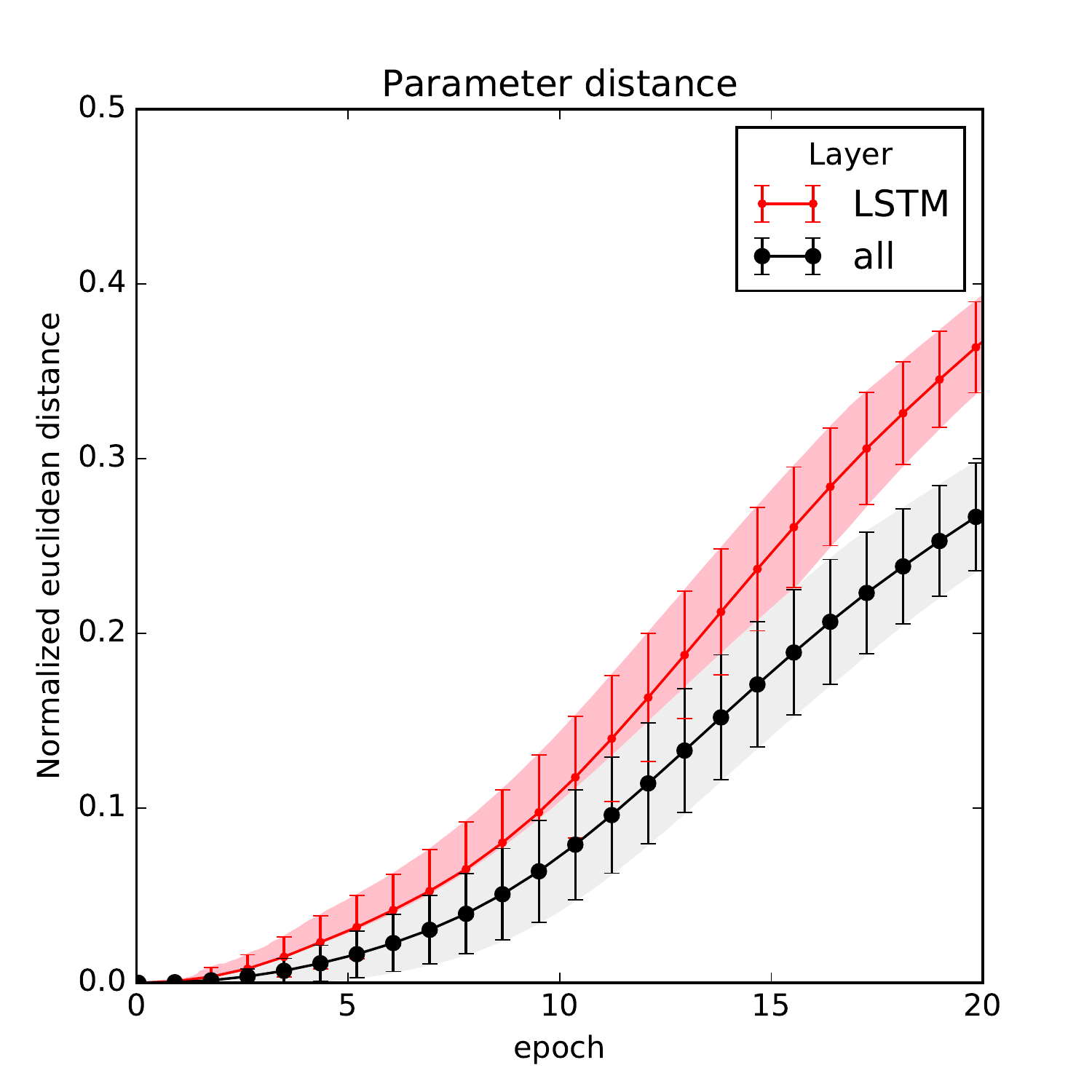}
\caption{Training on PTB dataset with LSTM architecture.}
\label{fig:ptb}
\end{figure}
}{
\begin{figure}[h!]
\includegraphics[width=0.5\textwidth]{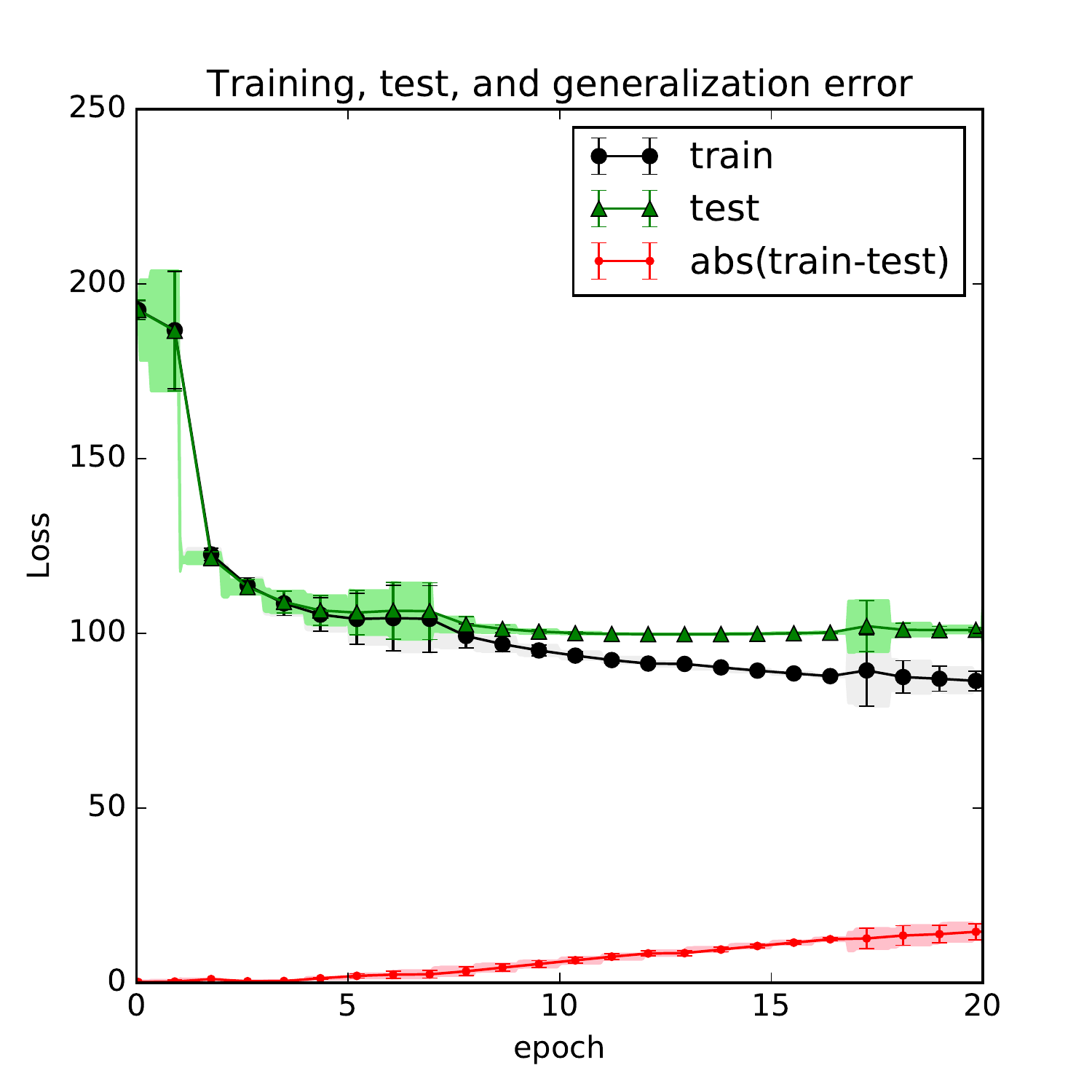}
\includegraphics[width=0.5\textwidth]{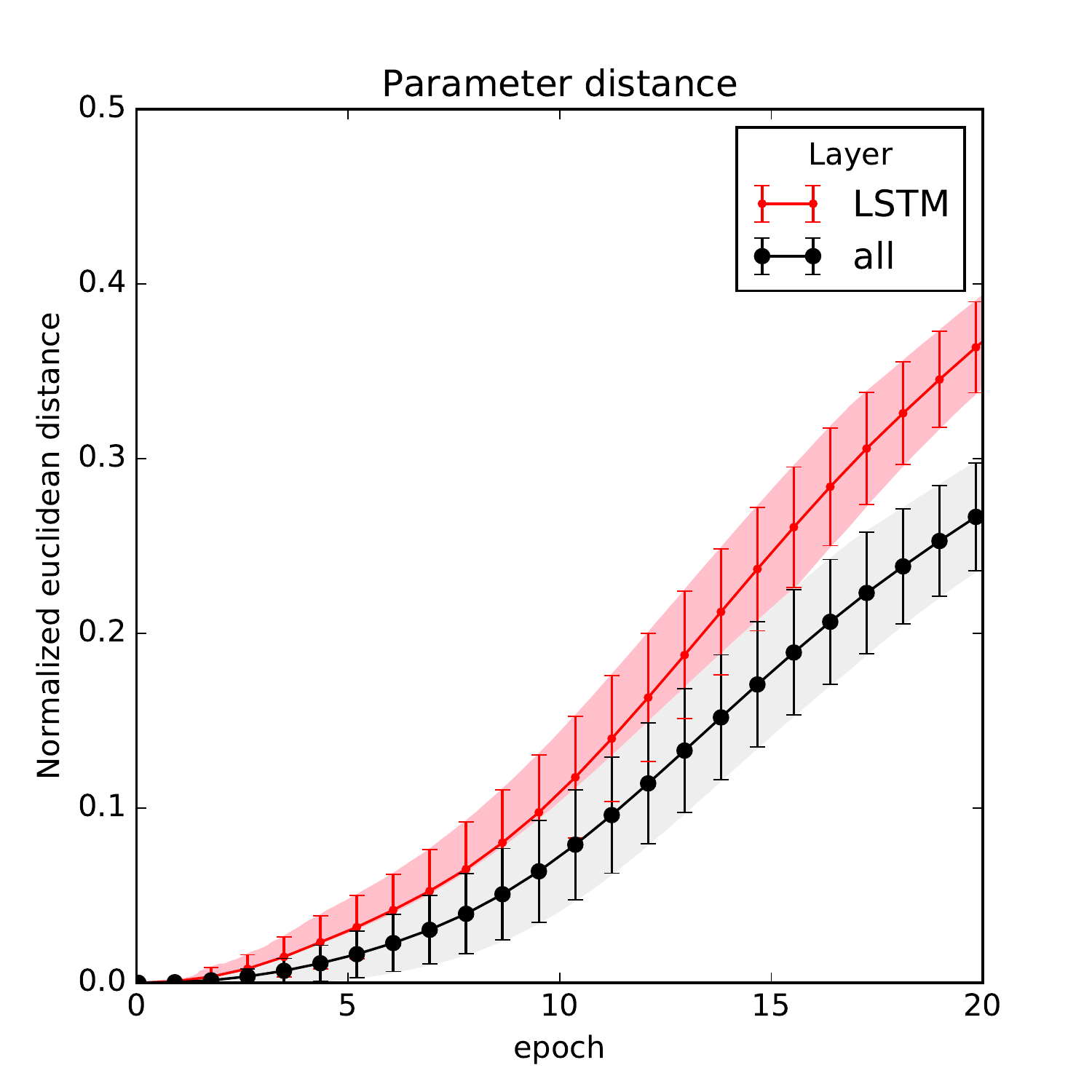}
\caption{Training on PTB data set with LSTM architecture.}
\label{fig:ptb}
\end{figure}
}

\section{Future Work and Open Problems}
Our analysis parts from much previous work in that we directly analyze the 
generalization performance of an algorithm rather than the solution of an 
optimization problem. In doing so we build on the toolkit usually used to prove
that algorithms converge in objective value.

This approach could be more powerful than analyzing optimality conditions, as
it may be easier to understand how each data point affects a procedure rather
than an optimal solution. It also has the advantage that the generalization
bound holds even if the algorithm fails to find a unique optimal solution as is
common in non-convex problems.

In addition to this broader perspective on algorithms for learning, there are
many exciting theoretical and empirical directions that we intend to pursue in
future work.  

\paragraph{High Probability Bounds.}  The results in this paper are all~\emph{in expectation}.  Similar to the well-known proofs of the stochastic gradient method, deriving bounds on the expected risk is relatively straightforward, but high probability bounds need more attention and care~\cite{Nemirovski09, Rakhlin11}.  In the case of stability, the standard techniques from Bousquet and Elisseeff require uniform stability on the order of $O(1/n)$ to apply exponential concentration inequalities like McDiaramid's~\cite{BousquetE02}.  For larger values of the stability parameter $\epsilon_{\mathrm{stab}}$, it is more difficult to construct such high probability bounds.  In our setting, things are further complicated by the fact that our algorithm is itself randomized, and thus a concentration inequality must be devised to account for both the randomness in the data and in the training algorithm.  Since differential privacy and stability are closely related, one possibility is to derive concentration via an algorithmic method, similar to the one developed by Nissim and Stemmer~\cite{NissimStemmer15}.  

\paragraph{Stability of the gradient method.}
Since gradient descent can be considered a ``limiting case'' of the stochastic gradient descent method, one can use an argument like our Growth Recursion Lemma (Lemma~\ref{lem:growth}) to analyze its stability.  Such an argument provides an estimate of $\epsilon_\mathrm{stab} \leq \tfrac{\alpha T}{n}$ where $\alpha$ is the step size and $T$ is the number of iterations. Generic bounds for convex functions suggest that gradient descent achieves an optimization error of $O(1/T)$.  Thus, a generalization bound of $O(1/\sqrt{n})$ is achievable, but at a computational complexity of $O(n^{1.5})$.  SGM, on the other hand, achieves a generalization of $O(1/\sqrt{n})$ in time $O(n)$.

In the non-convex case, we are unable to prove any reasonable form of stability
at all. In fact, gradient descent is not \emph{uniformly} stable as it does not enjoy the 
``burn-in'' period of SGM as illustrated in Figure~\ref{fig:hill}.
\begin{figure}
\includegraphics[width=0.5\textwidth,height=2in]{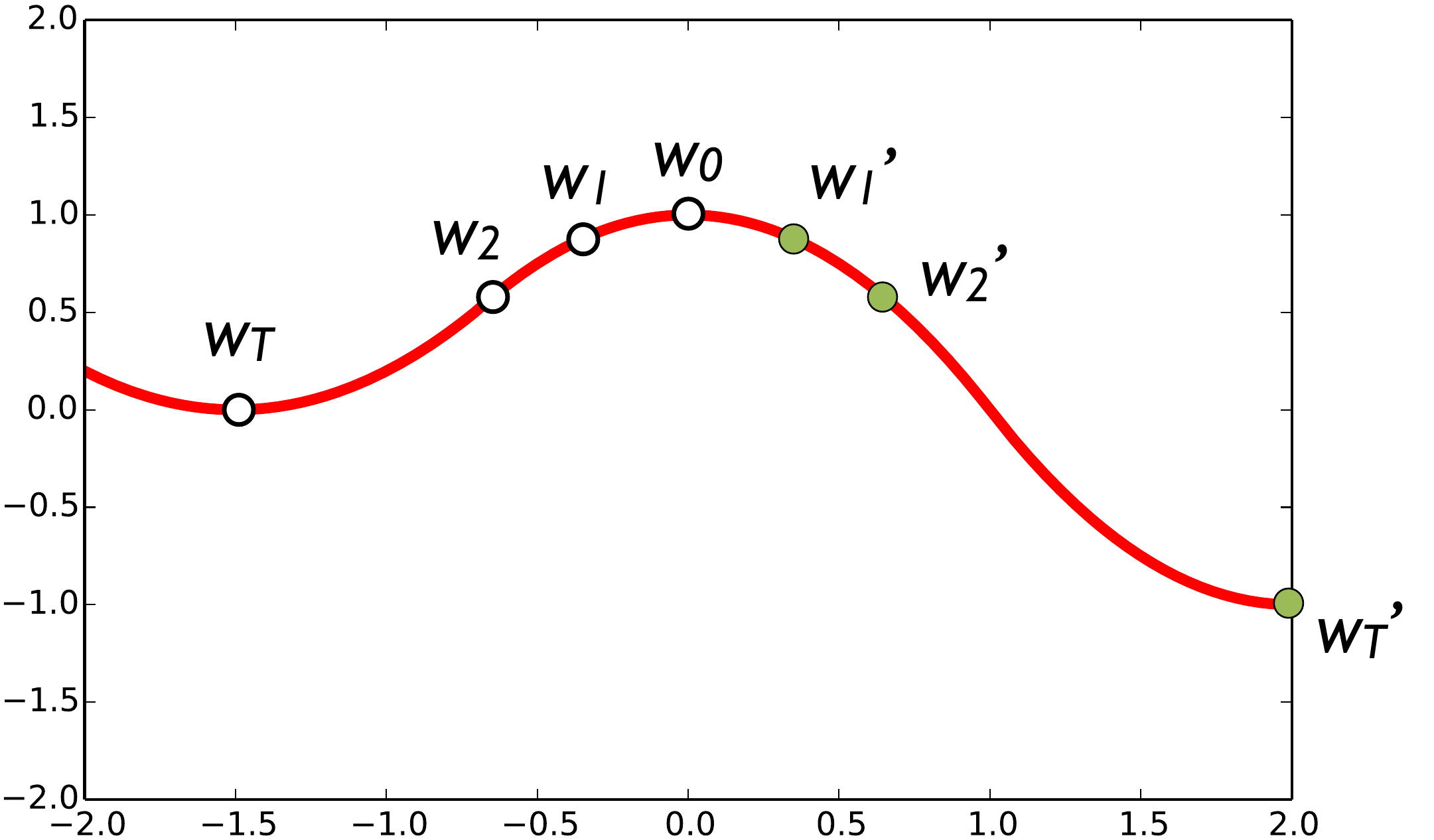}
\includegraphics[width=0.5\textwidth,height=2in]{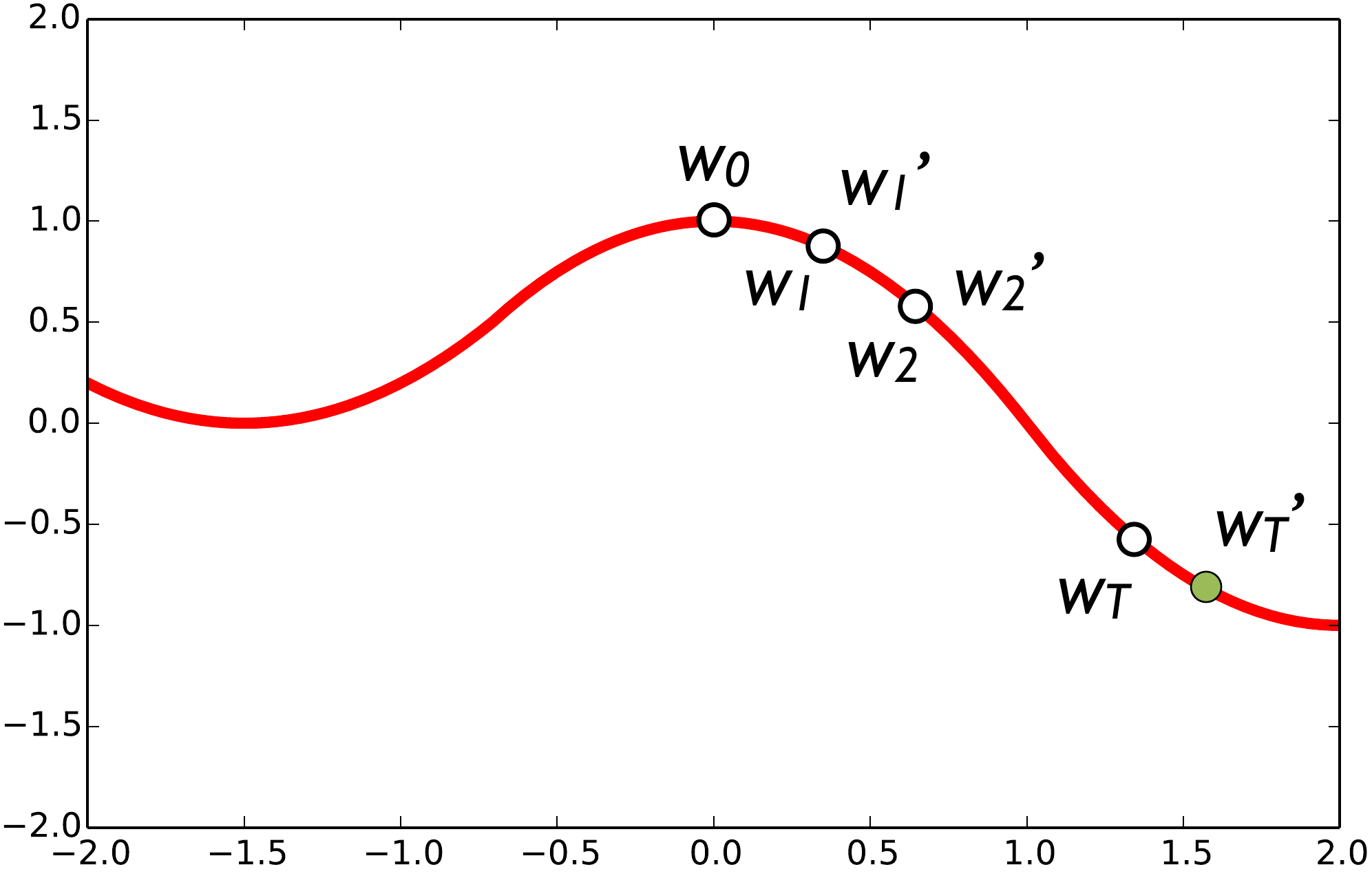}
\caption{Left: Gradient descent is not uniformly stable for non-convex functions. Right: SGM is stable due to ``burn-in'' period.}
\label{fig:hill}
\end{figure}
Poor
generalization behavior of gradient descent has been observed in practice, but
lower bounds for this approach are necessary to rule out a stable
implementation for non-convex machine learning.

\paragraph{Acceleration and momentum.}
We have described how many of the best practices in neural net training can be understood as stability inducing operations.  One very important technique that we did not discuss is~\emph{momentum}.  In momentum methods, the update is a linear combination of the current iterate and the previous direction.  For convex problems, momentum is known to decrease the number of iterations required by stochastic gradient descent~\cite{Lan12}.  For general nonlinear problems, is believed to decrease the number of iterations required to achieve low-training error~\cite{Rumelhart85,PolyakBook}.  However, it is not clear that momentum adds stability.  Indeed, in the case of convex optimization, momentum methods are less robust to noise than gradient methods~\cite{devolder2014first,Lessard14}.  Thus, it is possible that momentum speeds up training but adversely impacts generalization.

\paragraph{Model Selection.} 
Another related avenue that bridges theory and practice is using stability as a method for model selection.  In particular, our results imply that the models that train the fastest also generalize the best.  This suggests that a heuristic for model selection would be to run many different parameter settings and choose the model which results in the lowest training error most quickly.  This idea is relatively simple to try in practice, and ideas from bandit optimization can be applied to efficiently search with this heuristic cost~\cite{Jamieson15,karnin2013almost}.  From the theoretical perspective, understanding the sensibility of this heuristic would require understanding \emph{lower bounds} for generalizability.  Are there necessary conditions which state that models which take a long training time by SGM generalize less well than those with short training times?

\paragraph{High capacity models that train quickly.}  
If the models can be trained quickly via stochastic gradient, our results prove that these models will generalize.  However, this manuscript provides no guidance as to how to build a model where training is stable \emph{and} training error is low.  Designing a family of models which both has high capacity and can be trained quickly would be of significant theoretical and practical interest.

Indeed, the capacity of models trained in current practice steadily increases as
growing computational power makes it possible to effectively train larger
models. It is not uncommon for some models, such as large neural networks, to
have more free parameters than the size of the sample yet have rather small generalization error~\cite{krizhevsky2012imagenet,SzegedyInception}. In fact, sometimes increasing the model capacity even seems to \emph{decrease} the generalization error~\cite{NeyshaburTS15}.  Is it possible to understand this phenomena via stability?  How can we find models which provably both have high capacity and train quickly?

\paragraph{Algorithm Design.}
Finally, we note that stability may also provide new ideas for designing learning rules.  There are a variety of successful methods in machine learning and signal processing that do not compute an exact stochastic gradient, yet are known to find quality stationary points in theory and practice~\cite{bucklew1993weak}.  Do the ideas developed in this paper provide new insights into how to design learning rules that accelerate the convergence and improve the generalization of SGM?

\section*{Acknowledgements}
The authors would like to thank 
Martin Abadi,
Samy Bengio, 
Thomas Breuel,
John Duchi, 
Vineet Gupta, 
Kevin Jamieson,
Kenneth Marino,
Giorgio Patrini,
John Platt, 
Eric Price, 
Ludwig Schmidt,
Nati Srebro, 
Ilya Sutskever,
and Oriol Vinyals 
for their insightful feedback and helpful suggestions.

\bibliographystyle{abbrv}
\bibliography{../refs}

\appendix

\section{Elementary properties of convex functions}

\begin{proof}[Proof of Lemma~\ref{lem:smooth2expansive}]
Let $G=G_{f,\alpha}.$ By triangle inequality and our smoothness assumption,
\begin{align*}
\|G(v)-G(w)\|
& \le \|v-w\|+ \alpha \|\nabla f(w) - \nabla f(v)\|\\
& \le \|v-w\| + \alpha\beta\|w-v\| \\
& = (1+\alpha\beta)\|v-w\|\,. \qedhere
\end{align*}
\end{proof}

\begin{proof}[Proof of Lemma~\ref{lem:convex2expansive}]
Convexity and $\beta$-smoothness implies that the gradients are
co-coercive, namely
\begin{equation}\label{eq:co-coercive}
\langle\nabla f(v)-\nabla f(w),v-w\rangle 
\ge \frac1\beta\|\nabla f(v)-\nabla f(w)\|^2\,.
\end{equation}
We conclude that
\begin{align*}
	\|G_{f,\alpha}(v) - G_{f,\alpha}(w)\|^2 
&= \|v-w\|^2 - 2\alpha\langle \nabla f(v) - \nabla f(w), v-w \rangle 
+ \alpha^2 \|\nabla f(v)-\nabla f(w)\|^2\\
&\leq \|v-w\|^2-\left(\tfrac{2\alpha}{\beta}-\alpha^2\right) \|\nabla f(v)-\nabla f(w)\|^2\\
& \le \|v-w\|^2\,.\qedhere
\end{align*}
\end{proof}

\begin{proof}[Proof of Lemma~\ref{lem:sconvex2expansive}]
First, note that if $f$ is $\gamma$ strongly convex, then $\varphi(w) = f(w) -
\tfrac{\gamma}{2}\|w\|^2$ is convex with $(\beta-\gamma)$-smooth.  Hence, applying~\eqref{eq:co-coercive} to $\varphi$ yields the inequality
\[
\langle \nabla f(v) - \nabla f(w), v-w, \rangle \geq \frac{\beta \gamma}{\beta+\gamma} \|v-w\|^2
+\frac{1}{\beta+\gamma} \| \nabla f(v) - \nabla f(w)\|^2
\]
Using this inequality gives
\[
\begin{aligned}
\|G_{f,\alpha}(v)-G_{f,\alpha}(w)\|^2  &=
\|v-w\|^2 - 2 \alpha \langle \nabla f(v) - \nabla f(w), v-w, \rangle + \alpha^2 \| \nabla f(v) - \nabla f(w)\|^2\\
&\leq  \left(1-2 \frac{\alpha \beta \gamma}{\beta+\gamma}  \right)\|v-w\|^2 - \alpha\left(\frac{2}{\beta+\gamma} - \alpha\right) \| \nabla f(v) - \nabla f(w)\|^2\,.
\end{aligned}
\]
With our assumption that $\alpha \leq \frac{2}{\beta+ \gamma}$, this implies
\[
\|G_{f,\alpha}(v)-G_{f,\alpha}(w)\|  \leq \left(1-2 \frac{\alpha \beta \gamma}{\beta+\gamma}  \right)^{1/2} \|v-w\|\,.
\]
The lemma follows by applying the inequality $\sqrt{1-x} \leq 1-x/2$ which holds for $x \in [0,1]$.
\end{proof}

\begin{proof}[Proof of Lemma~\ref{lem:prox-expansive}]
This proof is due to Rockafellar~\cite{Rockafellar76}. Define
\begin{equation}\label{eq:prox-map-def}
P_\nu(w) = \arg \min_v \frac{1}{2\nu}\|w-v\|^2  + f(v)\,.
\end{equation}
This is the proximal mapping associated with $f$.
Define the map $Q_\nu(w):= w - P_\nu(w)$.  Then, by the optimality conditions
associated with~\eqref{eq:prox-map-def}, we have
\[
	\nu^{-1} Q_\nu(w) \in \partial f (P_\nu(w))\,.
\]
By convexity of $f$, we then have
\[
\langle P_\nu(v) - P_\nu(w), Q_\nu(v)-Q_\nu(w)\rangle \geq 0\,.
\]
Using this inequality, we have
\[
\begin{aligned}
		  \|v-w\|^2&=\|[P_\nu(v) - P_\nu(w)]+[Q_\nu(v)-Q_\nu(w)]\|^2\\
		  & = \|P_\nu(v) - P_\nu(w)\|^2 +2 \langle P_\nu(v)
				- P_\nu(w), Q_\nu(v)-Q_\nu(w)\rangle
				+ \|Q_\nu(v)-Q_\nu(w)\|^2\\
		&\geq \|P_\nu(v) - P_\nu(w)\|^2\,,
\end{aligned}
\]
thus completing the proof.
\end{proof}

\if{0}
\newpage

\section{Minibatches}

For a subset $B \subset S$, denote a \emph{minibatch} stochastic gradient by the average of $\nabla f(w;z)$ over the examples in the set $B$:
\[
	\nabla f(w;B) = \frac{1}{|B|} \sum_{i \in B} \nabla f(w;z)
\]

For a minibatch stochastic gradient update, we sample $b$ examples without replacement from the training set $S$ and update
\[
	w_{t+1} = w_t - \alpha \nabla f(w_t; B)
\]

To analyze the minibatch case, we need a slightly modified version of our Growth Recursion Lemma.
\begin{lemma}[Growth recursion for minibatches]
\label{lem:growth-mb}
Fix an arbitrary sequence of minibatch stochastic gradient updates $\nabla f(\cdot; B_1),\dots,\nabla f(\cdot; B_T)$ with batch size $b$ sampled from the set $S$.
Let the sequence $\nabla f(\cdot; B_1'),\dots,\nabla f(\cdot; B_T')$ where example $z_i$ is replaced by $z_i'$.
Let $\delta_t = \|w_t - w_t'\|$
Then, we have the recurrence relation
\begin{align*}
\delta_0 & = 0 \\
\delta_{t+1} & \le \begin{cases}
\eta\delta_{t} & \text{$B_t=B_t'$ and $w- \nabla f(w;B_t)$ is $\eta$-expansive}  \\
\eta \delta_{t}+2\tfrac{L}{b} & \text{$B_t\neq B_t'$ $\nabla f(w;z)$ is bounded by $L$ for all $w$ and $z$} \\
\end{cases}\tag{$t>0$}
\end{align*}
\end{lemma}
\begin{proof}
The first bound on $\delta_t$ follow directly from the assumption that
the minibatch SGD updates are from identical minibatches and then from the definition of expansiveness. For the second bound, we apply the triangle inequality
\begin{align*}
\delta_{t+1} =
&\| w_t - \alpha\nabla f(w_t;B) - w_t' + \alpha \nabla f(w_t';B_t')\| + \alpha \| \nabla f(w_t';B_t) -\nabla f(w_t';B_t')\|\\
\leq & \eta \delta_t +\frac{\alpha}{b} \| \nabla f(w_t';z_t) - \nabla f(w_t',z_t')\|\\
\leq & \eta \delta_t + \frac{2 \alpha L}{b}\,.
\end{align*}
\end{proof}

We additionally need a modified statement of our ``burn in'' Lemma:
\begin{lemma}
\label{lem:meta-mb}
Assume that the loss function $f(\cdot\,;z)$ is $L$-Lipschitz for all~$z.$ Let
$S$ and $S'$ be two samples of size $n$ differing in only a single example,
and denote by $w_T$ and $w_T'$ the output of $T$ steps of minibatch SGD with
batch size $b$ on $S$ and $S',$ respectively. Then, for every $z\in Z$
and every $t_0\in\{0,1,\dots,n\},$ under both the random update rule and the
random permutation rule, we have
\[
\E\left|f(w_T;z)-f(w_T';z)\right| \le \frac{b t_0}n +
L\E\left[\delta_T\mid\delta_{t_0}=0\right]\,.
\]
\end{lemma}
\begin{proof}
The proof is nearly identical to that of Lemma~\ref{lem:meta}.  The only difference is in the final probability bound.  Note that in both the permutation model and the random selection rule, the probability that a particular element is selected before time $t_0$ is bounded by $b t_0/n$.  This completes the proof.
\end{proof}

\subsection{Convex and Strongly Convex Cases}

For the convex and strongly convex cases, the stability bounds for mini-batches are identical to the bounds for single-example SGD.  The proofs are also almost the same,  but with a minor modification to the main recurrence relations.  We state these results and sketch the proofs here.

\begin{theorem}
\label{thm:convex-mb}
Assume that $f\colon \Omega\to[0,1]$ is a decomposable convex $\beta$-smooth
$L$-Lipschitz function and that we run minibatch SGD with step sizes $\alpha_t\le 2/\beta$ and batch sizes $b$ for $T$ steps.  Then, SGD has uniform stability with
\[
\epsilon_{\mathrm{stab}}
\le \frac{2L^2}{n}\sum_{t=1}^T\alpha_t\,.
\]
\end{theorem}

\begin{proof}
The proof is nearly identical to the proof of Theorem~\ref{thm:convex}.  The only thing that changes is~\eqref{eq:convex-recursion}.  Note that at step $t$, with probability $1-b/n,$ the
batch of examples selected by minibatch SGD is the same in both $S$ and $S'.$ In this case we have
that $G_t=G_t'$ and we can use the $1$-expansivity of the update rule $G_t$.  With probability $b/n$,  the replaced example is selected in which case we the second case in~\ref{lem:growth-mb}. We then conclude that for every $t,$
\begin{align*}
\E\delta_t 
 \le \left(1-\frac{b}{n}\right)\E\delta_t + \frac{b}{n}\E\delta_t  +
\frac{2\alpha_t L}n  = \E\delta_t + \frac{2L\alpha_t}{n}\,.
\end{align*}
The proof is now identical to the proof of Theorem~\ref{thm:convex}.
\end{proof}

\begin{theorem}
\label{thm:sconvex-mb}
Assume that $f\colon \Omega\to[0,1]$ is a decomposable 
$\gamma$-strongly convex $L$-Lipschitz $\beta$-smooth
function and that we run minibatch SGD with constant step size $\alpha\le 1/\beta$ and batch size $b$ for $T$ steps.
Then, SGD has uniform stability with
\[
\epsilon_{\mathrm{stab}}
\le \frac{2 L^2}{\gamma n}\,.
\]
\end{theorem}

\begin{proof}
The proof is  nearly identical to that of Theorem~\ref{thm:sconvex} with a slightly different
recurrence relation.  Just as in the convex case, the probabiility that the minibatches at time $t$ are identical is $1-b/n$.  Using Lemma~\ref{lem:growth}, linearity of expectation, and Lemma~\ref{lem:sconvex2expansive}, we then have
\begin{align*}
\E\delta_t
& \le \left(1-\frac{b}{n}\right)(1-\alpha \gamma)\E\delta_t + \frac{b}{n}(1-\alpha \gamma)\E\delta_t  +
\frac{2\alpha L}n  \\
& \le \left(1- \alpha \gamma \right)\E\delta_t    +
\frac{2\alpha L}n  \,.
\end{align*}
The proof is now identical to the proof of Theorem~\ref{thm:sconvex}.
\end{proof}

For decaying stepsizes, we actually find that generalizing the analysis for $b=1$ to $b>1$ yields an upper bound bound degrades as we increase our minibatch size.  However, this does not preclude a new analysis specially tailored to this minibatch setting.
\begin{theorem}
\label{thm:sconvex-decaying-mb}
Assume that $f\colon \Omega\to[0,1]$ is a decomposable $\gamma$-strongly convex 
$L$-Lipschitz $\beta$-smooth
function and that we run minibatch SGD with step sizes  
$\alpha_t= \frac{1}{\gamma t }$ and batch size $b$.
Then, SGD has uniform stability with
\[
	\epsilon_{\mathrm{stab}} \leq \frac{2L^2 + b\beta}{\gamma n}\,.
\]
\end{theorem}

\begin{proof}
Note that once $t>\frac{\beta}{\gamma(1-1/n)}$, the iterates are contractive with contractivity $(1-\alpha_t \gamma)$.   Thus, for $t>t_0 := \lceil \frac{\beta}{\gamma(1-1/n)}\rceil$, we have
\[
\begin{aligned}
	\delta_{t+1} &\leq (1-\tfrac{b}{n} )(1-\alpha_t \gamma)\delta_t + \frac{b}{n}((1-\alpha \gamma)\delta_t + \frac{1}{b} 2 \alpha_t L)\\
	&=(1-  \alpha_t \gamma ) \delta_t + \frac{2 \alpha_t L }{n}\\
	& = (1 - \frac{1}{t}) \delta_t + \frac{2  L }{\gamma t n} 
\end{aligned}
\]
The proof is now identical to that of Theorem~\ref{thm:sconvex-decaying}, but we apply Lemma~\ref{lem:meta-mb} instead of Lemma~\ref{lem:meta}.
\end{proof}

\subsection{Non-convex optimization}

For nonconvex optimization, we again see a significant worsening of our stability bounds.  Essentially, each minibatch costs the same as $b$ stochastic gradient steps.  Again the proof is almost identical to the case when $b=1$, so we provide only a sketch.

\begin{theorem}
Assume that $f\colon \Omega\to[0,1]$ is a decomposable $L$-Lipschitz
$\beta$-smooth function and that 
we run SGD for $T$ steps with monotonically non-increasing 
step sizes $\alpha_t\le c/t.$ Then, SGD has uniform stability with
\[
\epsilon_{\mathrm{stab}}
\le 
\frac{\beta c+1}{n}(2cL^2)^{\frac1{\beta c+1}}(bT)^{\frac{\beta c}{\beta c+1}}\,.
\]
In particular, if $\beta=O(1),$ $c=O(1),$ and $L=O(1),$
then 
\[
\epsilon_{\mathrm{stab}}
\le O\left(\frac{(bT)^{1-1/(\beta c+1)}}{n}\right)
\,.
\]
\end{theorem}

\begin{proof}
Let $S$ and $S'$ be two samples of size $n$ differing in only a single example.
Consider the gradient updates $G_1,\dots,G_T$ and $G_1',\dots,G_T'$ induced by
running SGD on sample $S$ and $S',$ respectively. Let $w_T$ and $w_T'$ denote
the corresponding outputs of SGD. 

By Lemma~\ref{lem:meta-mb}, we have for every $t_0\in\{1,\dots,n\},$
\begin{equation}\label{eq:nonconvex-diff-mb}
\E\left|f(w_T)-f(w_T')\right| \le \frac{b t_0}n +
L\E\left[\delta_T\mid\delta_{t_0}=0\right]\,,
\end{equation}
where $\delta_t=\|w_T-w_T'\|.$
To simplify notation, let $\Delta_t = \E\left[\delta_T\mid\delta_{t_0}=0\right].$
We will bound $\Delta_t$ a function of $t_0$ and then optimize for $t_0.$

Toward this goal, observe that at step $t,$ with probability $1-b/n,$ the
example selected by SGD is the same in both $S$ and $S'.$ In this case we have
that $G_t=G_t'$ and we can use the $(1+\alpha_t\beta)$-expansivity of the update rule $G_t$
which follows from our smoothness assumption via
Lemma~\ref{lem:smooth2expansive}. With probability $b/n$ the selected example is
different in which case we use that both $G_t$ and $G_t'$ are
$\alpha_tL$-bounded as a consequence of Lemma~\ref{lem:Lipschitz2bounded}.

Hence, we can apply Lemma~\ref{lem:growth} and linearity of expectation to
conclude that for every $t>t_0,$
\begin{align*}
\Delta_t 
& \le \left(1-\frac{b}{n}\right)(1+\alpha_t\beta)\Delta_t + \frac{b}{n}(1+\alpha_t \beta)\Delta_t  +
\frac{2\alpha_t L}{n}\\
& \leq \left(1 + \frac{c\beta}t\right)\Delta_t + \frac{2cL}{tn}\\
& \le \exp\left(\frac{c\beta}t\right)\Delta_t + \frac{2cL}{tn}.\\
\end{align*}
Here we used that $1+x\le\exp(x)$ for all $x.$

Using the fact that $\Delta_{t_0}=0,$ we can unwind this recurrence relation
from $T$ down to $t_0+1.$ This gives
\begin{align*}
\Delta_T &\leq \sum_{t=t_0+1}^T \left\{\prod_{k=t+1}^T
\exp\left( \tfrac{\beta c}{k}\right) \right\}  \frac{2c L}{t n}
\le \frac{2L}{\beta n} \left(\frac{T}{t_0}\right)^{\beta c}\,,
\end{align*}
Plugging this bound into~\eqref{eq:nonconvex-diff-mb},
we get
\[
\E\left|f(w_T)-f(w_T')\right| \le \frac{bt_0}n +
\frac{2L^2}{\beta n} \left(\frac{T}{t_0}\right)^{\beta c}\,.
\]
Letting $q=\beta c,$ the right hand side is minimized when
\[
t_0 = \left(2cL^2/b\right)^{\frac1{q+1}} T^{\frac{q}{q+1}}\,.
\]
This setting gives us 
\begin{align*}
\E\left|f(w_T)-f(w_T')\right| 
& \leq
\frac{\beta c+1}{n}(2cL^2)^{\frac1{\beta c+1}}(bT)^{\frac{\beta c}{\beta c+1}}\,.
\end{align*}
Since the bound we just derived holds for all $S$ and $S',$ we immediately get
the claimed upper bound on the uniform stability.
\end{proof}

\subsection{Gradient Descent}

Note that these mini-batch analyses all immediately give stability bounds for the  traditional gradient descent method.

\begin{corollary}
\label{thm:convex-gradient}
Assume that $f\colon \Omega\to[0,1]$ is a decomposable 
$L$-Lipschitz $\beta$-smooth  function. Let
\[
	\bar{f}(w)=\frac{1}{n} \sum_{i=1}^n f(w;z_i)
\]
denote the average of $f$ over the set of samples $S$.

\begin{enumerate}
\item With no additional assumptions on $f$, if we run gradient descent on $\bar{f}$ with step size $\alpha$ for  $T$ iterations, the resulting estimator $w_T$ has uniform stability
\[
	\epsilon_{\mathrm{stab}}
\le \min \left\{ 2L^2 T \alpha,(1+\alpha \beta)^{T+1} \frac{2 L^2}{\beta n} \right\}  \,.
\]

\item If $\bar{f}$ is convex and and we run gradient descent on $\bar{f}$ with step sizes $\alpha_t\le 2/\beta$ for  $T$ iterations.  Then, the resulting estimator $w_T$ has uniform stability with
\[
\epsilon_{\mathrm{stab}}
\le \frac{2L^2}{n}\sum_{t=1}^T\alpha_t\,.
\]
\item If $\bar{f}$ is $\gamma$-strongly convex and we run gradient descent with a constant stepsize $\alpha < 1/\beta$,  the resulting estimator has $w_T$ has uniform stability with
\[
\epsilon_{\mathrm{stab}}
\le \frac{2L^2}{\gamma n}\,.
\]
\end{enumerate}
\end{corollary}

\begin{proof}

For part 1, the first upper bound follows using the second case of Lemma~\ref{lem:growth}. For the other upper bound, we use the second case of Lemma~\ref{lem:growth-mb} and the formula for summing geometric series.

Parts 2 and 3 are immediate consequences of Theorems~\ref{thm:convex-mb} and~\ref{thm:sconvex-mb}.
\end{proof}

Note that the bound for the nonconvex case rather poor unless the step sizes
and number of iterations are both very small. While these are only upper
bounds on stability, they do suggest that proving stability for gradient
descent for nonconvex functions would be very challenging.  Indeed, it's
quite easy to construct an example where changing one example in the
training set can entirely change the point to which gradient descent
converges.
\fi

\end{document}